%% file: arxiv_main.tex
\documentclass{article}

\usepackage{subfigure}
\usepackage{fullpage}

\usepackage{amsmath,amssymb,amsfonts,amsthm}
\allowdisplaybreaks

\newtheorem{definition}{Definition}

\newcommand{\argmax}{\operatornamewithlimits{argmax}}

\newcommand{\E}{\mathbb{E}}
\newcommand{\I}{\mathbb{I}}
\newcommand{\R}{\mathbb{R}}

\newcommand{\bv}{\boldsymbol{v}}

\newcommand{\cA}{\mathcal{A}}
\newcommand{\cB}{\mathcal{B}}
\newcommand{\cC}{\mathcal{C}}
\newcommand{\cD}{\mathcal{D}}
\newcommand{\cE}{\mathcal{E}}

\newcommand{\cI}{\mathcal{I}}
\newcommand{\cJ}{\mathcal{J}}

\newcommand{\cM}{\mathcal{M}}
\newcommand{\cN}{\mathcal{N}}
\newcommand{\cO}{\mathcal{O}}
\newcommand{\cR}{\mathcal{R}}
\newcommand{\cS}{\mathcal{S}}

\newcommand{\cV}{\mathcal{V}}
\newcommand{\cP}{{\mathcal{P}}}

\usepackage{algorithm}
\usepackage{algorithmicx}
\usepackage{algpseudocode}
\algrenewcommand\algorithmicrequire{\textbf{Input:}}
\algrenewcommand\algorithmicensure{\textbf{Output:}}

\usepackage{ifthen}
\usepackage{xcolor}

\usepackage{hyperref}
\hypersetup{colorlinks=true, linkcolor=blue, citecolor=blue, anchorcolor=blue, urlcolor=blue}  

\usepackage{thm-restate}

\newcommand{\compilehidecomments}{true}

\ifthenelse{\equal{\compilehidecomments}{true}}{%
    \newcommand{\wei}[1]{}
    \newcommand{\haoyu}[1]{}
}{
    \newcommand{\wei}[1]{{\color{blue!50!black}  [\text{Wei:} #1]}}
    \newcommand{\haoyu}[1]{{\color{brown!60!black} [\text{Haoyu:} #1]}}
}

\newcommand{\compilefullversion}{true}
\ifthenelse{\equal{\compilefullversion}{false}}{%
    \newcommand{\OnlyInFull}[1]{}
    \newcommand{\OnlyInShort}[1]{#1}
}{%
    \newcommand{\OnlyInFull}[1]{#1}%
    \newcommand{\OnlyInShort}[1]{}%
}%

\mathchardef\mhyphen="2D
\newcommand{\Elim}{{\sf Elim\mhyphen NS}}

\title{{Online Second Price Auction with Semi-bandit Feedback \\
        Under the Non-Stationary Setting  }
    \author{Haoyu Zhao\\
        IIIS, Tsinghua University\\
        \texttt{zhaohy16@mails.tsinghua.edu.cn}
        \and
        Wei Chen\\
        Microsoft Research\\
        \texttt{weic@microsoft.com}
    }
}
\date{}

\begin{document}
    
\maketitle

\begin{abstract}
    In this paper, we study the non-stationary online second price auction problem. We assume that the seller is selling the same type of items in $T$ rounds by the second price auction, and she can set the reserve price in each round. 
    In each round, the bidders draw their private values from a joint distribution unknown to the seller. Then, the seller announced the reserve price in this round. Next, bidders with private values higher than the announced reserve price in that round will report their values to the seller as their bids.
    The bidder with the highest bid larger than the reserved price would win the item and she will pay to the seller the price equal to the second-highest bid or the reserve price, whichever
        is larger.
    The seller wants to maximize her total revenue during the time horizon $T$ while learning the distribution of private values over time. 
    The problem is more challenging than the standard online learning scenario since the private value distribution is non-stationary, meaning that the distribution of bidders' private values may change over time, and we need to use the \emph{non-stationary regret} to measure the performance of our algorithm. 
    To our knowledge, this paper is the first to study the repeated auction in the non-stationary setting theoretically. 
    Our algorithm achieves the non-stationary regret upper bound $\tilde\cO(\min\{\sqrt{\cS T}, \bar\cV^{\frac{1}{3}}T^{\frac{2}{3}}\})$, where $\cS$ is the number of switches in the distribution, and $\bar\cV$ is the sum of total variation, and $\cS$ and $\bar\cV$ are not needed to be known by the algorithm.
    We also prove regret lower bounds $\Omega(\sqrt{\cS T})$ in the switching case and $\Omega(\bar\cV^{\frac{1}{3}}T^{\frac{2}{3}})$ in the dynamic case, showing that our algorithm has nearly optimal \emph{non-stationary regret}.
\end{abstract}

\section{Introduction}\label{sec:intro}
As the Internet is rapidly developing, there are more and more online repeated auctions in our daily life, such as the auctions on the e-Bay website and the online advertisement auctions on Google and Facebook. 
Perhaps the most studied and applied auction mechanism is the online repeated second price auctions with a reserve price. 
In this auction format, a seller repeatedly sells the same type of items to a group of bidders. 
In each round $t$, the seller selects and announces a reserve price $r^{(t)}$ while the bidders draw their private values $\bv^{(t)}$ on the item from a joint value distribution, 
    which is unknown to the seller.
For each bidder $i$, if its private value $\bv_i^{(t)}$ is at least the reserve price $r^{(t)}$, she will submit her bid $\bv_i^{(t)}$ to the seller; otherwise she will
    not submit her bid since she would not win if her value is less than the announced reserve price.
After the seller collects the bids in this round (if any), she will give the item to the highest bidder, and collect from this winner the payment equal to the value of the second-highest bid or the reserve price, 
    whichever is higher.
If no bidder submits bids in this round, that means the reserve price the seller announced is too high, and the seller receives no payment.
Such repeated auctions are common in online advertising applications on search engine or social network platforms.
The seller's objective is to maximize her cumulative revenue, which is the total payment she collects from the bidders over $T$ rounds.
Since the seller does not know the private value distribution of the bidders, the seller has to adjust the reserve price over time, hoping to learn the optimal reserve price.

The above setting falls under the multi-armed bandit framework, where reserve prices can be treated as arms and payments as rewards. 
As in the multi-armed bandit framework, the performance of an online auction algorithm is measured by its {\em regret}, 
    which is the difference between the optimal reward that always chooses the best reserve price and the expected cumulative reward of the algorithm. 
When the distribution of private values does not change over time, results from \cite{bianchi2017algorithmic,zhao2019stochastic} can be applied to 
    solve the above problem, whereas the work in \cite{cesa2015regret} considers a somewhat different setting where the seller only gets the reward as the feedback but does not see the bids (full-bandit feedback) and the private value distribution of each bidder is i.i.d.

In real-world applications, however, the private value distribution of the bidders may likely change over time, e.g., some important events happen, which greatly influence the market perception.
When the private value distribution changes over time, the optimal reserve price will also change and there is no single optimal reserve value.
None of the above studies would work under this realistic setting, except resetting the algorithms by human intervention.
Since it is difficult to predict distribution changes, we prefer to have algorithms that could automatically detect distribution changes and adjust their actions accordingly, and still provide nearly optimal performance over the long run.
    
In this paper, we design the first online learning algorithm for online second price auction with non-stationary distributions of private values.
We assume that the private values of the bidders at time $t$ follow the joint distribution $\cD_t$, and we assume that $r^*_t$ is the best reserve price at time $t$. 
We use \emph{non-stationary regret} to measure the performance of the algorithm, which is the difference between the expected cumulative reward of the best reserve prices at each round and the expected cumulative reward of the algorithm. 
We use two quantities to measure the changing of the distributions $\{\cD_t\}_{t\le T}$: switchings and total variation. 
The number of switchings is defined as $\cS := 1 + \sum_{t=2}^{T}\I\{\cD_t \neq \cD_{t-1}\}$, and the total variation is given as $\bar\cV := \sum_{t=2}^{T}||\cD_t -\cD_{t-1}||_{\text{TV}}$, where $||\cdot ||_{\text{TV}}$ denotes the total variation of the distribution and $T$ is the total time horizon (Section \ref{sec:model}).

In this paper, we provide an elimination-based algorithm that can achieve the \emph{non-stationary regret} of $\tilde\cO(\min\{\sqrt{\cS T}, \bar\cV^{\frac{1}{3}}T^{\frac{2}{3}}\})$ (Section \ref{sec:algorithm}). 
This regret bound shows that if the switchings or the total variations are not large (sublinear to $T$ in particular), our algorithm can still achieve sublinear \emph{non-stationary regret}. We give a proof sketch in Section \ref{sec:sketch} to show the main technical ideas of the regret analysis. 
We further show the non-stationary regret is lower bounded by $\Omega(\sqrt{\cS T})$ in the switching case, and lower bounded by $\Omega(\bar\cV^{\frac{1}{3}}T^{\frac{2}{3}})$
    in the dynamic case (Section \ref{sec:auction}), which means that our $\Elim$ algorithm achieves nearly optimal regret in the non-stationary environment. 
Moreover, our algorithm is parameter-free, which means that we do not need to know the parameters $\cS$ and $\bar\cV$ in advance and the algorithm is self-adaptive. 
Our main method is to reduce the non-stationary online auction problem into a variant of the non-stationary multi-armed bandit problem called {\em non-stationary one-sided full information bandit}, and solve this problem with some novel techniques. 

\OnlyInShort{Due to the space constraint, the detailed technical proofs are included in our full report~\cite{HC19nonstationary}, but the proof sketch covering all essential ideas are included in the main 
	text.}\OnlyInFull{The proof sketch covering all essential ideas are included in the main 
	text, and the detailed technical proofs are included in the appendix.
}

\subsection{Related Work}

\textbf{Multi-armed bandit: }Multi-armed bandit (MAB) problem is first introduced in
\cite{robbins1952bulletin}. MAB problems can be classified into stochastic bandits and the adversarial bandits. 
In the stochastic case, the reward is drawn from an unknown distribution, and in the adversarial case, the reward is determined by an adversary. 
Our model is a generalization of the stochastic case, as discussed below.
The classical MAB algorithms include UCB \cite{auer2002finite} and Thompson sampling \cite{thompson1933likelihood} for the
stochastic case and EXP3 \cite{auer2002nonstochastic} for the adversarial case. 
We refer to \cite{bubeck2012regret} for comprehensive coverage on the MAB problems.

\textbf{Non-stationary MAB: } Non-stationary MAB can be view as a generalization of the stochastic MAB, where the reward distributions are changing over time. 
The non-stationary MAB problems are analyzed mainly under two settings: 
    The first considers the switching case, where there are $\cS$ number of switchings in the distribution, and derives switching regret in terms of $\cS$ and $T$ \cite{garivier2011upper,wei2016tracking,liu2018change};
The second considers the dynamic case, where the distribution is changing continuously but the variation $\cV$ is bounded, and present dynamic regret in terms of $\cV$ and $T$\cite{gur2014stochastic,besbes2015nonstationary}.
However, most of the studies need to use $\cS$ or $\cV$ as algorithm parameters, which may not be easy to obtain in practice. 
Designing parameter-free algorithms has been studied in the full-information case \cite{luo2015achieving,jun2017online,zhang2018dynamic}.
    There are also several attempts to design parameter-free algorithms in the bandit case \cite{karnin2016multi,luo2018efficient,cheung2019learning}, but the regret bound is not optimal. 
A recent and innovative study \cite{auer2O19adaptively} solves the problem in the bandit case and achieves optimal regret. Then, \cite{chen2019anewalgorithm} significantly generalizes the previous work by extending it into the non-stationary contextual bandit and also achieves optimal regret.
Our study is the first one on the non-stationary one-sided full information bandit and its application to the online auction setting.

\textbf{Online auction: } 
For the online case where the private value distribution is unknown, \cite{cesa2015regret,bianchi2017algorithmic,zhao2019stochastic} consider different forms of the online second price auction. 
These studies assume that bidders truthfully follow their private value distributions, the same as we assume in this work.
\cite{mohri2015revenue} further considers the online second price auction with strategic bidders, which means that their bidding may not be truthful. 
\cite{roughgarden2016minimizing} studies the online second price auction with bidder specific reserve price. However, they need to use all the bidding information, and they also assume that the bidders are truthful.
For the offline case where the private value distribution is known, 
    the classical work by Myerson \cite{Myerson81} provides an optimal auction algorithm when the private value distributions of all bidders are independent and known, and the seller
could set different reserve prices for different bidders.

\section{Preliminary and Model}\label{sec:model}
In this section, we introduce the non-stationary online second price auction with semi-bandit feedback. We will also introduce the non-stationary regret to measure the performance of the algorithm. As mentioned before, we reduce the non-stationary online second price auction problem to a non-stationary bandit problem, which we called non-stationary one-sided full information bandit. We will also give the formal definition of the bandit problem and show the performance measurement for the corresponding bandit problem.

\begin{definition}[Non-stationary Online Second Price Auction]\label{def:non-stationary-auction}
    There are a fixed number of $n$ bidders and a seller, and the seller sells the same item in each round $t \in [T]$. In each round $t$, the seller sells the item through second price auction with reserve price $r^{(t)}$, where $r^{(t)}$ is chosen by the seller at the beginning of each round $t$ and is announced to the bidders before the bidders give their private values. The values of the bidders follow a distribution $\cD_t$ with support $[0,1]^n$ in round $t$, and the environment draws a vector of realized values for the bidders $\bv^{(t)} \sim \cD_t$. For each bidder $i\in [n]$, if her value $\bv^{(t)}_i \ge r^{(t)}$, she will report her value $\bv^{(t)}_i$ to the seller, otherwise she will not report her value and not attend the auction in this round.\footnote{We fully understand that in the repeated online second price auction, the bidder may not be truthful since she may participate in the auction in several rounds. However, this is out of the scope of the current paper. We will assume that the bidders are truthful in each round, and it's a good approximation in some cases.} The seller then dispatches the item using the second price auction with reserve price $r^{(t)}$. We assume that the distributions $\cD_t$ are generated \textbf{obliviously}, i.e. $\cD_t$ are generated before our algorithm starts, or equivalently, $\cD_t$ are generated independently to the randomness of $\cD_s$ for all $s\le t$ and the randomness of the algorithm.
\end{definition}

The performance of the reserve price in auction is always measured by the revenue:
$\cR(r^{(t)},\cD_t) := \E_{\bv\sim\cD_t}\left[\sum_{i=1}^n p_i(r^{(t)},\bv)\right]$,
where $p_i(r^{(t)},\bv)$ denote the money bidder $i$ needs to pay when the reserve price is $r^{(t)}$ and $\bv$ is the private value vector of the bidders is $\bv$.
In particular, if bidder $i$ has the highest bid among all bidders and its bid is also larger than the reserve price $r^{(t)}$, then $i$ pays the maximum value among all other bids and the reserve price and gets the auction item; otherwise the bidder $i$ pays nothing and does not get the item. 
 Note that if we fix a reserve price $r$, whether bidders with values less than $r$ report their values or not does not affect the revenue. 
 Given the revenue of a reserve price, we have the following definition for the non-stationary regret in the online second price auction.

\begin{definition}[Non-stationary Regret for Online Second Price Auction]
    The non-stationary regret of algorithm $\cA$ for the online second price auction is defined as follow,
    \vspace{-2mm}
    \[\text{Reg}_{\cA}^{SP} := \E\left[\sum_{t=1}^T (\cR(r^*_t,\cD_t) - \cR(r^{(t)},\cD_t))\right],\]
    where $r^*_t := \argmax_r \cR(r,\cD_t)$ and $r^{(t)}$ is the reserve price algorithm $\cA$ chooses in round $t$, and the expectation $\E[\cdot]$ is taken over all the randomness, including the randomness of the algorithm itself and the randomness of $\bv^{(1)},\ldots, \bv^{(t-1)}$ leading to the randomness in the selection of $r^{(t)}$.
\end{definition}

We now introduce the measurement of the non-stationarity. In general, there are two measurements of the change of the environment: the first is the number of the swichings $\cS$, and the second is the total variation $\bar\cV$. 
For any interval $\cI = [s,s']$, we define the number of switchings on $\cI$ to be $\cS_{\cI} := 1 + \sum_{t=s+1}^{s'}\I\{\cD_t \neq \cD_{t-1}\}$. As for the total variation, the formal defintion is given as $\bar\cV_{\cI} := \sum_{t=s+1}^{s'}||\cD_t -\cD_{t-1}||_{\text{TV}}$, where $||\cdot ||_{\text{TV}}$ denotes the total variation of the distribution. For convenience, we use $\cS$ and $\bar\cV$ to denote $\cS_{[1,T]}$ and $\bar\cV_{[1,T]}$. 

Next, we briefly discuss how to reduce the online second price auction to the one-sided full-information bandit: 
1) We can discretize the reserve price into $r_1,\dots,r_K$. Because the revenue of the second price auction is one-sided Lipschitz, when $K$ is large enough, the revenue of the best discretized reserve price should not make so much difference to that of the best reserve price on the whole domain.
2) The distribution of the value $\cD_t$ will induce a distribution of reward on $(r_1,\dots,r_K)$. More specifically, any private value vector $\bv^{(t)}\sim\cD_t$ will induce a reward vector $X^{(t)} = (X^{(t)}_1,\dots,X^{(t)}_K)$ for the discretized reserve price $r_1,\dots,r_K$, and the reward vector $X^{(t)}$ follows a distribution $\nu_t$.
3) At time $t$, because all bidders with values at least $r^{(t)}$ will report their values, we can compute the rewards for all $r \ge r^{(t)}$  given the specific private values larger than or equal to $r^{(t)}$. 
This gives us the following definition of the non-stationary one-sided full-information bandit. The formal reduction from the online auction to the bandit problem will be given in the proof of the Theorem \ref{thm:regret-auction}.

\begin{definition}[Non-stationary One-sided Full Information Bandit]\label{def:non-stationary-bandit}
        There is a set of arms $\{1,2,\dots,K\}$, and for each arm $a\in [K]$ at time $t$, it corresponds to an unknown distribution $\nu_{a,t}$ with support $[0,1]$, where $\nu_{i,t}$ is the marginal distribution of $\nu_t$ with support $[0,1]^K$. In each round $t$, the environment draws a reward vector $X^{(t)} = (X^{(t)}_1,\dots,X^{(t)}_K)$, where $X^{(t)}$ is drawn from distribution $\nu_t$. The player then chooses an arm $A_t$ to play, gains the reward $X^{(t)}_{A_t}$ and observes the reward of arms $A_t,A_t+1,\dots,K$, i.e. observes $X^{(t)}_{i},\forall i \ge A_t$. We assume that the distribution $\nu_t$ at each round $t$ is generated \textbf{obliviously}, i.e. $\nu_t$ are generated before the algorithm starts.
    \end{definition}
    We use $\mu_{a,t}$ to denote the mean of $X^{(t)}_a$, i.e. $\mu_{a,t} = \E[X^{(t)}_a]$. We also use $\mu^*_t = \max_a \mu_{a,t}$ to denote the mean of the best arm at time $t$. Then we have the following definition of the non-stationary regret.
    
    \begin{definition}[Non-stationary Regret]\label{def:regret}
        We use the following to denote the non-stationary regret of algorithm $\cA$.
        \vspace{-2mm}
        \[\text{Reg}_{\mathcal A} := \E\bigg[\sum_{t=1}^T\left(\mu^*_t - \mu_{A_t,t}\right)\bigg].\]
    \end{definition}

    For convenience, we will simply use regret to denote the non-stationary regret. 
    We now introduce the measurements for the non-stationarity for the one-sided bandit case. 
    Similar to the auction case, we have switchings $\cS$ and variation $\cV$. For any interval $\cI = [s,s']$, we define the number of switchings on $\cI$ to be $\cS_{\cI} := 1 + \sum_{t=s+1}^{s'}\I\{\nu_t \neq \nu_{t-1}\}$. As for the sum of variation, the formal definition is given as $\cV_{\cI} := \sum_{t=s+1}^{s'}\max_{a}|\mu_{a,t} -\mu_{a,t-1}|$, which sums up the max difference of mean in each round. For convenience, we use $\cS$ and $\cV$ to denote $\cS_{[1,T]}$ and $\cV_{[1,T]}$. 
    Note that the number of switchings in the bandit case is the same as that of the auction case, so we reuse the notations, and 
    the variation definition in the bandit case uses the sum of the maximal differences in the consecutive mean vectors instead of the sum of total variations in the auction case, so we use notation
    $\cV$ instead of $\bar\cV$ for differentiation.
    The variation $\cV$ defined for the bandit case is consistent with the variation defined in other non-stationary bandit papers.
    
    We will use Switching Regret to denote the non-stationary regret in the switching case, and dynamic regret to denote the non-stationary regret in the dynamic case.

\section{Algorithm}\label{sec:algorithm}
In this section, we present our algorithm $\Elim$ for the non-stationary one-sided full-information bandit problem and its regret bounds. 
The algorithm for the online auction problem can be easily derived from $\Elim$, as outlined in Section~\ref{sec:model}, 
    and we present its regret bound in Theorem \ref{thm:regret-auction}.

\begin{algorithm*}[!th]
    \caption{$\Elim$}
    \label{alg:nonstationaryelim}
    \begin{algorithmic}[1]
        \Require Total time horizon $T$, total number of arms $K$. Parameters $C_1,C_2$.
        \State $t\leftarrow 1, \ell \leftarrow 1, \tau_{\ell} \leftarrow t$. \Comment{$\tau_{\ell}$ is the starting time of epoch $\ell$.}
        \State $\cM\leftarrow\phi, a_{\text{min}} \leftarrow 1,\cE\leftarrow\phi$. 
        \State Let $\hat\mu_{a}[t_1,t_2)$ denote the empirical mean of arm $a$ in the time interval $[t_1,t_2)$.
        \While{$t \le T$}
        \State \noindent\fbox{\parbox{2.9in}{Step 1. Randomly select the exploration phases}}
        \If{$\cM\neq\phi$}
            \State $\Delta_{t,\text{min}} \leftarrow \min_{(g,e,\bv)\in \cM}g$.
        \EndIf
        \State Let $d_i \leftarrow 2^{-i}$ for every $i\in\mathbb N$, and $I_t \leftarrow \max\{i:8d_i\ge \Delta_{t,\text{min}}\}$. \Comment{We define the notation $d_i$ for convenience.}
        \State For every $i \le I_t$, independently add pair $\left(d_i,\left[t,t+\lceil\frac{C_2\ln (KT^3)}{d_i^2}\rceil\right)\right)$ into $\cE$ with probability $p_{\ell,i} = d_i\sqrt{\frac{\ell+1}{T}}$.  
        \label{algline:exploration-prob}
        \State (Let $\cE_t $ and $\cM_t$ be the values of $\cE$ and $\cM$ respectively at this point, to be used in the proof)
        \State \noindent\fbox{\parbox{2.9in}{Step 2. Choose an action to play}}
        \If{$\exists (d,\cI)\in\cE$ such that $t\in\cI$} \Comment{Choosing the arm based on if $t$ is in an exploration phase.}
            \State $d_{\max,t} \leftarrow \max_{(d,\cI)\in\cE, t\in \cI}d$.
            \State Play arm $A_t \leftarrow a_{\text{exp}} = \min\{k:\exists (g,e,\bv)\in\cM, k = e, g\le 8d_{\max,t}\}$ and observe the reward $X^{(t)}_a$ for all $a\ge a_{\text{exp}}$.\label{algline:exploration-arm}
        \Else
            \State Play arm $A_t \leftarrow a_{\text{min}}$ and observe the reward $X^{(t)}_a$ for all $a\ge a_{\text{min}}$.
        \EndIf
        \State \noindent\fbox{\parbox{2.9in}{Step 3. Perform the elimination process}}
        \While{$\exists \sigma \ge \tau_{\ell}, a>a_{\text{min}}$ such that $\hat\mu_a[\sigma,t+1) - \hat\mu_{a_{\text{min}}}[\sigma,t+1) > \sqrt{\frac{C_1\ln(KT^3)}{t+1-\sigma}}$}
            \State Let $\bv$ be a vector with length $K$.
            \State Let $b$ be the arm such that $\hat\mu_b[\sigma,t+1) - \hat\mu_{a_{\text{min}}}[\sigma,t+1)$ is maximized.
            \State $g \leftarrow \hat\mu_b[\sigma,t+1) - \hat\mu_{a_{\text{min}}}[\sigma,t+1)$, $e \leftarrow a_{\text{min}}$, and $\bv_i\leftarrow \hat\mu_{i}[\sigma,t+1)$ for all $i\ge a_{\min}$.
            \State $\cM \leftarrow \cM\cup\{(g,e,\bv)\}, \ a_{\text{min}} \leftarrow a_{\text{min}}+1$.
        \EndWhile
        \State \noindent\fbox{\parbox{2.9in}{Step 4. Perform the non-stationarity check}}
        \If{$\exists (d,[t',t+1))\in\cE, (g,e,\bv)\in\cM,a\ge e$ such that $g \le 8d$ and $|\hat\mu_a[t',t+1) - \bv_{a}| > \frac{d}{4}$}
            \State $\ell\leftarrow\ell+1,\cM\leftarrow\phi,\cE\leftarrow\phi,a_{\text{min}}\leftarrow 1,\tau_{\ell}\leftarrow t+1$.
        \EndIf
        \State $t\leftarrow t+1$.
        \EndWhile
    \end{algorithmic}
\end{algorithm*}

Our algorithm $\Elim$ borrows ideas from \cite{zhao2019stochastic} and \cite{auer2O19adaptively}. \cite{zhao2019stochastic} introduce an elimination-based algorithm for the one-sided full-information bandit, and \cite{auer2O19adaptively} present an elimination-based algorithm to adaptively follow the best arm in the switching case without knowing the number of switches $\cS$. 
Our algorithm is a non-trivial combination of these ideas, and our innovation highly depends on the feedback structure of the one-sided bandit problem.
The algorithm is given in Algorithm~\ref{alg:nonstationaryelim}.

Generally speaking, our algorithm maintains a set $\cE$ to record the exploration phases
     for the adaptive detection of the
    dynamic changes in the distribution, and a set $\cM$ to record the information when an arm is eliminated. 
If we were dealing with the stationary case where the distribution of arms does not change, 
    after observing arms for enough times, we can eliminate an empirically sub-optimal arm, and with high probability, the eliminated arm is indeed sub-optimal.
However, in the non-stationary case, the optimal arm is changing, and thus we need to properly add exploration phases to observe the eliminated arms with some probability. 
When we detect that the distribution indeed has changed from these exploration phases, the algorithm starts a new epoch and resets $\cE$ and $\cM$ to empty sets.
\footnote{We mark the actual values of $\cE$ and $\cM$ in each round as $\cE_t$ and $\cM_t$ in the algorithm, to be used in our analysis.
}

Set $\cM$ records the information at the time when an arm is eliminated. Each element $(g,e,\bv)\in\cM$ is a tuple, 
    where $g\in\R$ records the empirical gap, which is the difference of the empirical
    means of the empirically optimal arm and that of the eliminated arm $a_{\min}$; 
    $e = a_{\min}$ records the index of the eliminated arm; 
    and $\bv_k$ for $k\ge a_{\min}$ 
    records the empirical mean of arm $k$ when the arm $e$ is eliminated ($\bv\in\R^K$).
    An exploration phase is a pair $(d,\cI)$ where $d = 2^{-k}$ and interval $\cI \subseteq [T], |\cI| = \Theta(\frac{1}{d^2})$. 
    Each such phase is stored independently into $\cE$ with a probability (in line~\ref{algline:exploration-prob} of Step 1).
    The purpose of these exploration phases is to re-examine arms that have been eliminated to detect possible changes in the distribution, with
        $\cI$ indicating the range of rounds for an exploration.
    Intuitively, if there is no change in the distribution, such an exploration would pay an extra regret.
    To control this extra regret, we use $d$ to indicate the per-round regret 
        that such an exploration could tolerate, and the length of $\cI$ is controlled to be $\tilde{\cO}(1/d^2)$ to bound the total regret.
    
%

At each round, 
Our algorithm $\Elim$ has the following four steps.
In Step 1, we randomly add exploration phases into the set $\cE$. 
We set $p_{\ell,i} = d_i\sqrt{\frac{\ell+1}{T}}$ to be the probability to add an exploration phase $(d_i,[t,t+\lceil\frac{C_2\ln (KT^3)}{d_i^2}\rceil))$ into $\cE$ in epoch $\ell$ at time $t$. This probability is chosen carefully, not too small to omit the non-stationarity, and not too big to induce large regret.

In Step 2, we choose the action to play. 
If the current round $t$ is not in any exploration phase, then we will play the arm that is not eliminated and has the smallest index. If $t$ is in an exploration phase $(d,\cI) $, we will find the maximum value $d_{\max,t} = \max_{(d,\cI)\in\cE, t\in \cI}d$. 
We will play arm $A_t \leftarrow a_{\text{exp}} = \min\{k:\exists (g,e,\bv)\in\cM, k = e, g \le 8d_{\max,t}\}$ and observe the reward $X^{(t)}_a$ for all $a\ge a_{\text{exp}}$. 
This arm selection in the exploration phase guarantees that the arm we play would induce the regret of at most $\cO(d_{\max,t})$ per round if the distribution has not changed.

In Step 3, we perform arm elimination when the proper condition holds. 
In particular, when we find an arm is empirically sub-optimal among the remaining arms, we eliminate this arm in this epoch. 
When an arm is eliminated, the algorithm will add an tuple $(g,e, \bv)$ into the set $\cM$ to store the information at this point, 
    where $g$ stores the empirical gap with the best arm, $e$ stores the index of the eliminated arm, and for $k \ge e$ $\bv_k$ stores the empirical mean of arm $k$. 

In Step 4, we apply the non-stationarity check. 
At the end of an exploration phase, we check that if there is a tuple $(g,e,\bv)\in\cM$ and an arm $a\ge e$, such that 
    the gap between the current empirical mean of arm $a$ during the exploration phase and the stored empirical mean $\bv_a$
     is $\Omega(g)$.
If so, it means that the empirical mean has a significant deviation indicating a change in distribution,
    and thus we will start a new epoch to redo the entire process from scratch. 

The algorithm incorporates ideas from \cite{auer2O19adaptively,zhao2019stochastic}, and its main novelty is related to the maintenance and use of 
    set $\cM$ in arm selection (Step 2), arm elimination (Step 3) and stationarity check (Step 4), 
    which make use of the feedback observation to balance the exploration and exploitation.

Now, we use a simple example to illustrate how the $\Elim$ algorithm detects the distribution changes in the switching case. 
Suppose that we have three arms.
At first, arm 1 always outputs $0$, arm 2 always outputs $0.45$, and arm 3 always outputs $0.5$. 
Then arm 1 will be eliminated first, and the tuple $(g,e,\bv) = (0.5,1,(0,0.45,0.5))$ will be stored in $\cM$, where $g = 0.5$ is the empirical gap 
    between the means of arm 1 and the empirically best arm 3. 
Next arm 2 will be eliminated, and the algorithm will store $(0.05,2,(?,0.45,0.5))$ in $\cM$, where $?$ means that the value at that position has no meaning. 
At this point, the algorithm may have randomly selected many exploration phases, but they all fail to start a new epoch since the distribution does not change and non-stationarity 
    would not be detected. 
Then suppose that at round $t$, the distribution changes, and arm 1 will output $1$ from now on and thus becomes the best arm. 
Suppose that after round $t$, we randomly select an exploration phase with $d = 2^{-5}$, and in this exploration phase, we will play arm 2 
    but not arm 1 (since $0.05 \le 8* 2^{-5} < 0.5$), and thus we will still not detect the non-stationarity of arm 1. 
However, when we randomly select an exploration phase with $d=0.5$ in step 1 (perhaps in a later round), 
    we will play arm 1 according to the key selection criteria for arm exploration in line~\ref{algline:exploration-arm} of step 2.
This would allow us to observe the distribution change on arm 1 in the exploration phase and then start a new epoch, which will restart the algorithm from scratch by playing arm 1 again.

The following two theorems summarize the regret bounds of algorithm $\Elim$ in the switching case and the dynamic case for the one-sided full-information bandit.

\begin{restatable}[Switching Regret]{theorem}{thmswitchingregret}\label{thm:switching-regret}
    Suppose that we choose parameters $C_1 \ge 2048, C_2 \ge 32$, then the algorithm $\Elim$ has regret in the switching case bounded by $\tilde\cO(\sqrt{\cS T})$, where $\tilde\cO(\cdot)$ hides the polynomial factor of $\log K$ and $\log T$.
\end{restatable}

\begin{restatable}[Dynamic Regret]{theorem}{thmdynamicregret}\label{thm:dynamic-regret}
    Suppose that we $C_1 \ge 8192, C_2 \ge 128$, and suppose that the variation is not too small ($\cV = \Omega(1)$). Then the algorithm $\Elim$ has regret in the dyanmic case bounded by $\tilde\cO(\cV^{\frac{1}{3}} T^{\frac{2}{3}})$, where $\tilde\cO(\cdot)$ hide the polynomial factor of $\log K$ and $\log T$.
\end{restatable}

As outlined in Section~\ref{sec:model}, $\Elim$ can be easily adapted to solve the online second price auction problem by discretizing the reserve price.
The following theorem provides the regret bound of $\Elim$ on solving the online second price auction problem.

\begin{restatable}[Regret for Online Second Price Auction]{theorem}{thmregretauction}\label{thm:regret-auction}
    For every $0\le k \le \lceil\sqrt{T}\rceil$, let $r_k = \frac{k}{\lceil\sqrt{T}\rceil}$, and we only set reserve price $r^{(t)}\in\{r_1,\dots,r_{\lceil\sqrt{T}\rceil}\}$. Each time we set reserve price $r^{(t)} = r_{A_t}$ and get all the private value $v^{(t)}_i \ge r^{(t)}$, we compute the reward $X^{(t)}_k$ for all $k\ge A_t$ and receive the reward $X^{(t)}_{A_t}$. Then we apply our algorithm $\Elim$ and set $C_1,C_2$ appropriately, and the regret is bounded by
    \[\text{Reg}_{\cA}^{SP} \le \tilde\cO(\min\{\sqrt{\cS T}, \bar\cV^{\frac{1}{3}}T^{\frac{2}{3}}\}),\]
    where we assume that $\bar\cV = \Omega(1)$ is not too small.
\end{restatable}

\section{Proof Sketch for the Regret Analysis}\label{sec:sketch}
In this section, we will give a proof sketch of the regret analysis in the switching case (Theorem \ref{thm:switching-regret}) and the dynamic case (Theorem \ref{thm:dynamic-regret}). In general, we first give a proof in the switching case, and then we reduce the dynamic case into the switching case. 
The proof strategy in the dynamic case is nearly the same as that in the switching case, and we will briefly discuss how to do the reduction.

\subsection{Proof Sketch of Theorem \ref{thm:switching-regret}}

Generally speaking, our proof strategy for Theorem \ref{thm:switching-regret} is to define several events (Definitions \ref{def:sampling-nice},\ref{def:exploration-phase},\ref{def:procedure-nice},\ref{def:playing-bad-arm}), and decompose the regret by these events. We show that each term in the decomposition is bounded by $\tilde\cO(\sqrt{\cS T})$.

\begin{restatable}[Sampling is nice]{definition}{defsamplingnice}\label{def:sampling-nice}
    We say that the sampling is nice if for every interval $\cI\subseteq [T]$ and every arm $a$, we have
    \[\frac{1}{|\cI|}\biggr|\sum_{t\in\cI}X^{(t)}_a -\sum_{t\in\cI}\mu_{a,t}\biggr| < \sqrt{\frac{\ln(KT^3)}{2|\cI|}},\]
    where $|\cI|$ is the length of interval $\cI$. We use $\cN^s$ to denote this event. We use $\cN^s_t$ to denote the event when the above inequality holds for all $\cI\subseteq [t]$.
\end{restatable}

\begin{restatable}{definition}{defexplorationphase}\label{def:exploration-phase}
    We use $\cP_t$ to denote the event such that $t$ is in an exploration phase, i.e. $\exists (d,\cI)\in\cE_t \text{ such that } t\in\cI$.
\end{restatable}


\begin{restatable}[Records are consistent]{definition}{defprocedurenice}\label{def:procedure-nice}
    We say that the records are consistent at time $t$ if for every $(g,e,\bv)\in\cM_t$, for every arm $a \ge e$, we have
    $|\mu_{a,t}-\bv_{a}| \le \frac{g}{4}$.
    We use $\cC_t$ to denote this event.
\end{restatable}

We have the following definition when $\cC_t$ doesn't happen. 

\begin{restatable}[Playing bad arm]{definition}{defplayingbadarm}\label{def:playing-bad-arm}
    Let $b_t$ denote the smallest index of an arm such that $\exists (g,e,\bv)\in\cM_{t}$, $e = b_t$ and there exists $a\ge e, |\bv_{a}-\mu_{a,t}|>\frac{g}{4}$, i.e.
    \[b_t = \min\left\{e:(g,e,\bv)\in\cM_t,\exists a\ge e,|\bv_{a}-\mu_{a,t}|>\frac{g}{4}\right\}.\]
    We use $\cB_t$ to denote the event $\{A_t \ge b_t\}$.
\end{restatable}
Generally speaking, $b_t$ is the smallest index of an eliminated arm such that the recorded mean when $b_t$ is eliminated induces the event $\lnot \cC_t$. 

Based on the above definitions, we decompose the regret into four mutually exclusive events 
    and bound the regret for each event in the order of $\tilde\cO(\sqrt{\cS T})$.
These four event cases are listed below, where the first three are when the sampling is nice, and the last case is when sampling is not nice.

{\em Case 1}: $\cN^s\land \cC_t\land \lnot \cP_t$.
This means that the sampling is nice, the records are consistent at time $t$, and round $t$ is not in an exploration phase.
The regret should be bounded in this case, since when $\cC_t$ happens, the distribution does not change much and it is also not in an exploration phase
    (Lemma \ref{lem:first-term-switching}). 
    
{\em Case 2}: $\cN^s\land \cC_t\land \cP_t$ or $\cN^s\land \lnot \cC_t\land\lnot \cB_t$.
The sampling is still nice. 
When $\cC_t\land \cP_t$ is true, round $t$ is in an exploration phase and the records are consistent, meaning that the current arm means have not deviated much from the
    records.
In this case, similar as discussed before, the definition of the exploration phase $(d,\cI)$ and the setting in line~\ref{algline:exploration-arm} guarantee that the arm
    explored would not have a large regret.
When  $\lnot \cC_t\land\lnot \cB_t$ is true, we first claim that $\lnot \cC_t\land\lnot \cB_t$ implies $\cP_t$.
This is because if the records are not consistent (i.e. $\lnot \cC_t$) but $A_t < b_t$ (i.e. $\lnot \cB_t$), it means $A_t$ played in round $t$
    has smaller index than $b_t$, but $b_t$ is an eliminated arm according to Definition~\ref{def:playing-bad-arm}, and thus arm $A_t$ must be played
    due to exploration.
Next, since $A_t < b_t$, the arm played is not a bad arm with a large gap, so its regret is still bounded (Lemma \ref{lem:secondthird-term-switching}).

{\em Case 3}: $\cN^s\land \lnot \cC_t\land \cB_t$.
The sampling is nice, the records are not consistent, and in round $t$ we play a bad arm with a large gap between the current mean and the recorded mean.
Although the regret in this case cannot be bounded by $\cO(g)$ where $(g,A_t,\bv)\in\cM_t$, the key observation is that, due to the random selection of the exploration phase, we will observe the non-stationarity (since $\cC_t$ does not happen and $\cB_t$ happens) with some probability, and the expected regret can be bounded (Lemma \ref{lem:forth-term-switching}).

{\em Case 4}: $\lnot \cN^s$. The sampling is not nice, which is a low probability event, and its regret can be easily bounded by a constant
    (Lemma \ref{lem:fifth-term-switching}).

\begin{restatable}{lemma}{lemfirsttermswitching}\label{lem:first-term-switching}
    \begin{align*}
    &\E\left[\sum_{t=1}^T\left(\mu^*_t - \mu_{A_t,t}\right)\cdot\I\left\{\cN^s\land \cC_t\land \lnot \cP_t\right\}\right]
    \le 2\cS + 2(\sqrt{C_1} + \sqrt{2})\sqrt{\ln(KT^3)}\sqrt{2\cS T}.
    \end{align*}
\end{restatable}

The proof of Lemma \ref{lem:first-term-switching} is similar to the analysis in \cite{zhao2019stochastic} and can be viewed as a generalization of the original proof. The key difference is that in the proof of Lemma \ref{lem:first-term-switching}, we divide the interval into
\[[1,T] = [s_1,e_1]\cup [s_2,e_2]\cup\cdots\cup[s_{\cS},e_{\cS}],\]
and we sum the regret in each interval first, and get the regret in each interval to be $\tilde\cO(\sqrt{e_i - s_i + 1})$. 
Then we sum them up and show that the regret is in the order of $\tilde\cO(\sqrt{\cS T})$.

\begin{restatable}{lemma}{lemsecondthirdtermswitching}\label{lem:secondthird-term-switching}
    \begin{align*}
    &\E\left[\sum_{t=1}^T\left(\mu^*_t - \mu_{A_t,t}\right)\cdot\I\left\{\cN^s\land \cC_t\land \cP_t\right\}\right] +\E\left[\sum_{t=1}^T\left(\mu^*_t - \mu_{A_t,t}\right)\cdot\I\left\{\cN^s\land \lnot \cC_t\land\lnot \cB_t\right\}\right] \\
    \le& \left(C_2\ln (KT^3)\sqrt{(\cS+1)T} + 2\sqrt{\frac{\cS+1}{T}}\right) \times \left(3-\log_2 \sqrt{\frac{C_1\ln(KT^3)}{T}}\right).
    \end{align*}
\end{restatable}
This lemma bounds the regret when $\cB_t$ does not happen and $t$ is in an exploration phase. 
In this case,
    we show that the number of different lengths $d$ of exploration phases $(d,\cI)$ can be bounded by $polylog(K,T)$.
Then, we show that the regret induced by the specific length exploration phase is bounded by $\tilde\cO(\sqrt{\cS T})$.
Finally, we combine the previous argument and apply the union bound to show that the total regret considered is bounded by $\tilde\cO(\sqrt{\cS T})$.

\begin{restatable}{lemma}{lemforthtermswitching}\label{lem:forth-term-switching}
    \begin{align*}
    &\E\left[\sum_{t=1}^T\left(\mu^*_t - \mu_{A_t,t}\right)\cdot\I\left\{\cN^s\land \lnot \cC_t\land \cB_t\right\}\right]
    \le 24\sqrt{(\cS+1)T}+24\sqrt{C_2\ln(KT^3)\cS T}.
    \end{align*}
\end{restatable}
This lemma bound the regret when $\cB_t$ happens, and this lemma is the most technical one. 
The proof strategy is similar to \cite{auer2O19adaptively}, which partitions the total time horizon into several intervals with identical distribution, 
    and applies a two-dimensional induction from back to front. As discussed before, the regret in this case in each round cannot be bounded by $\cO(g)$ where $(g,A_t,\bv)\in\cM_t$.
However due to the random selection of the exploration phases, with some probability, we will observe the non-stationarity
    (since $\cC_t$ does not happen and $\cB_t$ happens), and the expected regret can be bounded.

Finally, by a simple application of the high probability result on $\cN^s$, we can get the following lemma.

\begin{restatable}{lemma}{lemfifthswitching}\label{lem:fifth-term-switching}
    $\E\left[\sum_{t=1}^T\left(\mu^*_t - \mu_{A_t,t}\right)\cdot\I\left\{\lnot\cN^s\right\}\right] \le 2$.
\end{restatable}

Combining these lemmas together, we complete the proof of Theorem \ref{thm:switching-regret}.

\subsection{Proof Sketch of Theorem \ref{thm:dynamic-regret}}

In this part, we briefly introduce how to reduce the dynamic case to the switching case. The proof is an imitation of the proof strategy of Theorem \ref{thm:switching-regret}. 
Although the means can be changing at every time $t\in [1,T]$, we can approximately divide them into several sub-intervals such that in each interval, the change of means is not large. Recall that for interval $\cI = [s,s']$, $\cV_{\cI} := \sum_{t=s+1}^{s'}\max_{a}|\mu_{a,t} -\mu_{a,t-1}|$ and we use $\cV := \cV_{[1,T]}$. We have the following lemma,

\begin{restatable}[Interval Partition \cite{chen2019anewalgorithm}]{lemma}{lemintervalpartition}\label{lem:interval-partition}
    There is a way to partition the interval $[1,T]$ into $\cI_1\cup\cI_2\cup\cdots\cup\cI_{\Gamma}$ such that $\cI_i \cap \cI_j = \phi$, and for any $i\le\Gamma$, $\cV_{\cI} \le \sqrt{\frac{C_3}{|\cI_i|}}$ and $\Gamma \le \left(2T/C_3\right)^{1/3}\cV^{\frac{2}{3}}+ 1$.
\end{restatable}

Suppose that we have a partition shown in the above lemma. We construct a new instance such that $\mu'_{a,t} = \frac{1}{|\cI_j|}\sum_{s\in\cI_j}\mu_{a,s}$ for all $j\le\Gamma$ and all $t\in\cI_j$, i.e. we take the average mean of each interval and make them all the same.

Generally speaking, the dynamic regret can be bounded by the sum of 2 parts: the switching regret of the new instance and the difference between the switching regret of the new instance and the dynamic regret. As for the first part, since $\Gamma \le \left(2T/C_3\right)^{1/3}\cV^{\frac{2}{3}}+ 1$, we know that the switching regret can be bounded by $\tilde\cO(\sqrt{\Gamma T}) = \tilde\cO(\cV^{\frac{1}{3}}T^{\frac{2}{3}})$. As for the difference between the 2 regret, since $|\mu_{a,t} - \mu'_{a,t}| \le \cV_{\cI_j}$ for $t\in\cI_j$, we sum up all $t$, we know that the difference is bounded by $\cO(\sum_{j}\sqrt{|\cI_j|}) = \cO(\sqrt{\Gamma T}) = \cO(\cV^{\frac{1}{3}}T^{\frac{2}{3}})$. Combine them together we complete the proof.

\subsection{Proof Sketch of Theorem \ref{thm:regret-auction}}
In the proof of Theorem \ref{thm:regret-auction}, we first show that the online second price auction has one-sided Lipschitz property, and thus discretizing the reserve price will not lead to a large regret. Next, we briefly discuss why discretizing the reserve price can lead to a one-sided full information bandit instance, and then it is easy to show that the regret can be bounded by $\tilde\cO(\sqrt{\cS T})$ in the switching case. To bound the regret in the dynamic case, we only have to set up the connection between the total variation $\bar\cV$ in the online auction and the variation $\cV$ in the bandit problem. The bridge between these two variables can be set up easily by the definition and property of total variation $||\cdot||_{\text{TV}}$.

\section{Lower Bounds for Online Second Price Auction in Non-stationary Environment}\label{sec:auction}
In this section, we show that for the online second price auction problem, 
    the regret upper bounds achieved by $\Elim$ is almost tight, by giving 
    a regret lower bound of $\Omega(\sqrt{\cS T})$ for the switching case, and 
    a lower bound of $\Omega(\bar\cV^{\frac{1}{3}}T^{\frac{2}{3}})$ for the dynamic case.


\begin{restatable}{theorem}{thmlowerboundswitching}\label{thm:lower-bound-switching}
    For any algorithm, and any $\cS > 0$, there exists a set distributions of bids $\cD_1,\dots,\cD_T$ where $\cS = 1 + \sum_{t=1}^{T-1}\I\{\cD_t \neq \cD_{t+1}\}$ is the number of switchings of the distribution and the non-stationary regret is at least $\Omega(\sqrt{\cS T})$. Moreover for any algorithm and any $\bar\cV \ge 1$, there exists $\cD_1,\dots,\cD_T$ where $\sum_{t=2}^{T}||\cD_t -\cD_{t-1}||_{\text{TV}} \le\cO(\bar\cV)$, such that the regret is at least $\Omega(\bar\cV^{\frac{1}{3}}T^{\frac{2}{3}})$.
\end{restatable}

Our theorem is based on the following result in \cite{cesa2015regret}. 

\begin{restatable}[Theorem 2 of \cite{cesa2015regret}]{proposition}{proplowerboundstationary}\label{prop:lower-bound-stationary}
    For any deterministic algorithm, there exists a distribution of bids operating with two bidders and the stationary regret is at least $\Omega(\sqrt{T})$.
\end{restatable}

The above proposition shows that in the full-information case, any deterministic algorithm will have stationary regret lower bounded by $\Omega(\sqrt{T})$ for the online second price auction problem. Generally speaking, we divide the time interval into $\cS$ segments, each with length $\frac{T}{\cS}$. We construct an instance such that the regret in each segment is $\Omega(\sqrt{T/\cS})$, and the total non-stationary regret sums up to be $\Omega(\sqrt{\cS T})$.

As for the regret in the dynamic case, the proof is very similar. We also divide the time horizon into $\Theta(\bar\cV^{\frac{2}{3}}T^{\frac{1}{3}})$ segments, and the total variation between the distribution of adjacent segments is bounded by $\left(\bar\cV/T\right)^{\frac{1}{3}}$.

\section{Conclusion and Further Work}\label{sec:conclusion}
We study the non-stationary online second price auction with the ``semi-bandit'' feedback structure in this paper. We reduce it into the non-stationary one-sided full-information bandit and show an algorithm $\Elim$ that solves the problem. Our algorithm is parameter-free, which means that we do not have to know the switchings $\cS$ and the variation $\cV$ in advance. Our algorithm is also nearly optimal in both cases. There are also some future directions to explore:

First, in this work, we consider the online auction with ``semi-bandit'' feedback, where all the bidders with private values exceeding or equaling the reserve price will report their private values. 
We can also consider the ``full-bandit'' feedback where the seller only gets the reward in each round but does not observe the private values and design parameter-free algorithms to solve it in the non-stationary case.
Second, in this work we use the second price auction and assume that the bidders are truthful.
We can also study how to generalize this non-stationary result into the strategic bidders' case or the other auction formats such as the generalized second price auction.

\clearpage
\bibliographystyle{plain}
\bibliography{ref}
\clearpage

\OnlyInFull{\input{appendix.tex}}

\end{document}

%% file: appendix.tex
\appendix

\section{Proof of Theorem \ref{thm:switching-regret} and \ref{thm:dynamic-regret}}
In this section, we give the detailed proof of Theorem \ref{thm:switching-regret} and \ref{thm:dynamic-regret}. We first present some key observations and lemmas, which are helpful in both the switching and dynamic case. Then we give the detailed proof for the switching case(Theorem \ref{thm:switching-regret}). Next, we give the proof for the dynamic case(Theorem \ref{thm:dynamic-regret}). However, we will not give a detailed proof, since the proof is very similar to that of Theorem \ref{thm:switching-regret}, and we will point out the difference.

First, we need the following definition and probability bound. The definition(Definition \ref{def:sampling-nice}) is straight-forward, and the probability bound follows directly from the Hoeffding's Inequality and union bound.

\defsamplingnice*

\begin{restatable}{lemma}{lemsamplingnice}\label{lem:samplingnice}
    We have the following probability bound,
    \[\Pr\{\lnot \cN^s\} \le \frac{2}{T}, \Pr\{\lnot \cN^s_t\} \le \frac{2}{T}.\]
\end{restatable}

\begin{proof}
    From the Hoeffding's inequality, we have for any interval $\cI\subseteq [T]$ and any arm $a$,
    \begin{align*}
        \Pr\left\{\frac{1}{|\cI|}\biggr|\sum_{t\in\cI}X^{(t)}_a -\sum_{t\in\cI}\mu_{a,t}\biggr| < \sqrt{\frac{\ln(KT^3)}{2|\cI|}}\right\}
        \le& 2\exp\left(-2\cdot |\cI| \frac{\ln(KT^3)}{2|\cI|}\right)\\
        =& \frac{2}{KT^3}.
    \end{align*}
    Then from the union bound, there are at most $T^2$ possible intervals $\cI\in [T]$ and we have
    \begin{align*}
        \Pr\{\lnot \cN^s\}
        =&\Pr\left\{\exists \cI,a,\frac{1}{|\cI|}\biggr|\sum_{t\in\cI}X^{(t)}_a -\sum_{t\in\cI}\mu_{a,t}\biggr| < \sqrt{\frac{\ln(KT^3)}{2|\cI|}}\right\}\\
        \le&\sum_{k,\cI}\Pr\left\{\frac{1}{|\cI|}\biggr|\sum_{t\in\cI}X^{(t)}_a -\sum_{t\in\cI}\mu_{a,t}\biggr| < \sqrt{\frac{\ln(KT^3)}{2|\cI|}}\right\} \\
        \le& \sum_{k,\cI}\frac{2}{KT^3} \\
        \le& \frac{2}{T}.
    \end{align*}
    Then
    \[\Pr\{\lnot \cN^s_t\} \le \frac{2}{T}\]
    follows directly since $\cN^s \subseteq \cN^s_t$.
\end{proof}

The next observation is not hard to prove, but it is one of the key observations of the proof. The observation highly depends on the feedback structure of the one-sided full-information bandit problem.

\begin{restatable}[]{lemma}{lemexplorationobservation}\label{lem:exploration-observation}
    Suppose that an exploration phase $\cI = [t',t+1)$ ends at time $t$ where $(d,\cI)\in\cE_t$. For any arm $a$ such that $a\ge\min\{e:(g,e,\bv)\in\cM_t,g\le 8d\}$, it is observed for $t+1-t'$ times during the exploration phase $\cI = [t',t+1)$.
\end{restatable}

\begin{proof}
    Suppose that time $t$ is in epoch $\ell$, then we know that $t' \ge \tau_{\ell}$, since $\cE_{\tau_{\ell}-1} = \phi$ then we only add $[t_1,t_2]$ into the exploration set for $t_1 \ge \tau_{\ell}$.
    
    Because the algorithm add $(d,\cI=[t',t])$ into $\cE_{t'-1}$(the original set is $\cE_{t'-1}$ and the new set is $\cE_{t'}$), we know that there exists $(g,e,\bv)\in\cM_{t'}$ such that $g\le 8d$. This is due to the fact that from the definition of our algorithm, we have $\Delta_{t',\min} = \min_{(g,e,\bv)\in\cM_{t'}}g$ and $8d \ge 8 d_{I_{t'}} \ge \Delta_{t',\min}$.
    
    Next, we can observe that for any $s\in [t',t]$, the number $\min\{e:(g,e,\bv)\in\cM_{s},g\le 8d\}$ remains the same, because we eliminate arms from small index to large index, and the arms eliminated in the interval $[t',t]$ must have index larger than $\min\{e:(g,e,\bv)\in\cM_{t'},g\le 8d\}$.
    
    Let $a_0 = \min\{e:(g,e,\bv)\in\cM_{t'},g\le 8d\}$, and we only have to show that in every time $s\in[t',t]$, we play the arm $A_s \le a_0$, which is also true since from the definition of the algorithm, we have $d \le d_{\text{max},s}$ and then
    \begin{align*}
        &\{\bv:(g,e,\bv)\in\cM_{t'},g\le 8d\}  \subseteq \{\bv:(g,e,\bv)\in\cM_{t'},g\le 8d_{\text{max},s}\}.
    \end{align*}\
\end{proof}

We also have the following definition to describe a time that is in an exploration phase.

\defexplorationphase*

\subsection{Switching Regret}
As for the parameters $C_1,C_2$ in the algorithm, we choose $C_1 \ge 2048, C_2 \ge 32$ in the switching regret analysis.

The next lemma shows that, with high probability, the number of epochs in our algorithm is at most $\cS$(the number of switchings).

\begin{restatable}{lemma}{lemnumberofepochsswitching}\label{lem:number-of-epochs-switching}
    When $\cN^s$ happens, we have at time $T$, $\ell \le \cS$, i.e., the number of epochs will not exceed the number of switchings.
\end{restatable}

\begin{proof}
    We partition the time interval $[1,T]$ into $\cS$ intervals with the same distribution. We set
    \[[1,T] = [s_1,e_1]\cup [s_2,e_2]\cup\cdots\cup[s_{\cS},e_{\cS}],\]
    where $s_1 = 1,e_{\cS} = T$, $s_{i+1} = e_i + 1$ for all $i\le \cS-1$, and $\nu_t = \nu_{t'}$ for all $t,t'$ in the same interval. We only have to show that, if $\cN^s$ happens and epoch $\ell$ starts at time $t$ in interval $[s_i,e_i]$, epoch $\ell+1$ will not start in the interval $[s_i,e_i]$.
    
    We prove by contradiction, suppose that epoch $\tau_{\ell},\tau_{\ell+1}\in [s_i,e_i]$. Since epoch $\ell$ ends in time interval $[s_i,e_i]$, we know that from the definition of algorithm(Step 4), $\exists t\in [s_i,e_i], (d,[t',t+1))\in\cE_t, (g,e,\bv)\in\cM'_t,a\ge e$ such that $g \le 8d$ and $|\hat\mu_a[t',t+1) - \bv_{a}| > \frac{d}{4}$, where $\cM'_t$ denote the set $\cM$ in time $t$ just before Step 4. Moreover, $\tau_{\ell+1} = t+1$. From now on, we will fix the variables  $d,a,t,g,e,\bv$. However we will show that when $\cN^s$ happens, $|\hat\mu_a[t',t+1) - \bv_{a}| \le \frac{d}{4}$.
    
    First we will show that $\min\{e^*:(g^*,e^*,\bv^*)\in\cM_t,g^*\le 8d\} = \min\{e^*:(g^*,e^*,\bv^*)\in\cM'_t,g^*\le 8d\}$. Because $(d,[t',t+1))\in\cE_t$, we know that $t' \ge \tau_{\ell}$, and when $(d,[t',t+1))$ is added into $\cE$, there exists $(g',e',\bv')\in\cM_{t'}$ such that $g' \le 8d$. Then same as the argument in Lemma \ref{lem:exploration-observation}, since $[t',t]$ is contained in epoch $\ell$, we only add elements into the set $\cM$, and we have $\cM_{t'} \subseteq \cM_t \subseteq \cM'_{t}$ and the added elements do not affect the minimum $\min\{e^*:(g^*,e^*,\bv^*)\in\cM,g^*\le 8d\}$.
    
    Then from Lemma \ref{lem:exploration-observation}, we know that arm $a$ has been observed for $t-t'+1 = \lceil\frac{C_2\ln (KT^3)}{d^2}\rceil$ times in the interval $[t',t]$, because $a \ge e \ge \min\{e^*:(g^*,e^*,\bv^*)\in\cM'_t,g^*\le 8d\} = \min\{e^*:(g^*,e^*,\bv^*)\in\cM_t,g^*\le 8d\}$ and from the definition of $\cN^s$(Definition \ref{def:sampling-nice}), we have
    \begin{align*}
        |\hat\mu_a[t',t+1)-\mu_{a,t}| \le& \sqrt{\frac{\ln(KT^3)}{2(t-t'+1)}}\\
        \le& \sqrt{\frac{\ln(KT^3)}{2\lceil\frac{C_2\ln (KT^3)}{d^2}\rceil}}\\
        \le& \sqrt{\frac{\ln(KT^3)d^2}{2C_2\ln (KT^3)}}\\
        \le& \frac{d}{\sqrt{2C_2}}.
    \end{align*}
    Then we consider $\bv_{a}$. Suppose that the vector $\bv$ is added into the set $\cM$ at time $s$, and we know that there exists $\sigma$ such that
    \[g > \sqrt{\frac{C_1\ln(KT^3)}{s+1-\sigma}},\]
    and from $\cN^s$, we have
    \begin{align*}
        |\bv_{a} - \mu_{a,s}| =& |\hat\mu_{a}[\sigma,s+1) - \mu_{a,s}| \\
        < & \sqrt{\frac{\ln(KT^3)}{2(s-\sigma+1)}}\\
        < & \frac{1}{\sqrt{2C_1}}g \\
        \le & \frac{8d}{\sqrt{2C_1}}.
    \end{align*}
    Then from the choice of parameters $C_1,C_2$ that $C_1 \ge 2048,C_2 \ge 32$, and $s,t\in [s_i,e_i]$, we have $\mu_{a,s} = \mu_{a,t}$ and thus
    \begin{align*}
        |\hat\mu_a[t',t+1) - \bv_{a}|
        \le& |\hat\mu_a[t',t+1)-\mu_{a,t}| + |\bv_{a} - \mu_{a,s}|\\
        <& \frac{d}{\sqrt{2C_2}} + \frac{8d}{\sqrt{2C_1}} \\
        \le& \frac{d}{8} + \frac{d}{8} \\
        =& \frac{d}{4}.
    \end{align*}
    Then we conclude the proof of this lemma.
\end{proof}

To derive the switching regret, we need the following definitions.


\defprocedurenice*

\defplayingbadarm*

Then given these definitions, we have the following lemma for decomposing the regret.

\begin{restatable}[Switching Regret Decomposition]{lemma}{lemdecomposingswitching}\label{lem:decomposing-switching}
    The regret can be decomposed into the following way,
    \begin{align*}
    \E\left[\sum_{t=1}^T\left(\mu^*_t - \mu_{A_t,t}\right)\right]
    =& \E\left[\sum_{t=1}^T\left(\mu^*_t - \mu_{A_t,t}\right)\cdot\I\left\{\cN^s\land \cC_t\land \lnot \cP_t\right\}\right]\\
    &\quad+ \E\left[\sum_{t=1}^T\left(\mu^*_t - \mu_{A_t,t}\right)\cdot\I\left\{\cN^s\land \cC_t\land \cP_t\right\}\right]\\
    &\quad +\E\left[\sum_{t=1}^T\left(\mu^*_t - \mu_{A_t,t}\right)\cdot\I\left\{\cN^s\land \lnot \cC_t\land\lnot \cB_t \right\}\right]\\
    &\quad +\E\left[\sum_{t=1}^T\left(\mu^*_t - \mu_{A_t,t}\right)\cdot\I\left\{\cN^s\land \lnot \cC_t\land \cB_t\right\}\right]\\
    &\quad +\E\left[\sum_{t=1}^T\left(\mu^*_t - \mu_{A_t,t}\right)\cdot\I\left\{\lnot\cN^s\right\}\right].
    \end{align*}
\end{restatable}
The lemma is easy to prove. We just have to notice that all of the indicate variables add to $1$ in all cases, i.e. for all $t$, we have
\begin{align*}
    &\I\left\{\cN^s\land \cC_t\land \lnot \cP_t\right\} + \I\left\{\cN^s\land \cC_t\land \cP_t\right\} + \I\left\{\cN^s\land \lnot \cC_t\land\lnot \cB_t \right\} + \I\left\{\cN^s\land \lnot \cC_t\land \cB_t\right\} + \I\left\{\lnot\cN^s\right\} = 1.
\end{align*}

Then, we show the proof of each term.

\lemfirsttermswitching*

\begin{proof}
    The first observation is that, when $\cC_t$ happens, all the arms $a < A_t$ cannot be the optimal arm. The observation is based on the fact that: All the arms $a < A_t$ are eliminated in the current epoch, and suppose that event $\cC_t$ happens. For any eliminated arm $a$ and its corresponding vector $\bv(e = a)$, we know that there exists arm $b$ such that $\bv_{b} - \bv_{a} = g$, and we have
    \begin{align*}
    \mu_b - \mu_a \ge& \bv_{b} - \bv_{a} - |\bv_{b} - \mu_b| - |\bv_{a} - \mu_a| \\
    \ge& g - 2\times \frac{g}{4} \\
    \ge& \frac{g}{2}.
    \end{align*}
    Like before, we partition the time interval $[1,T]$ into $L$ intervals with the same distribution. We set
    \[[1,T] = [s_1,e_1]\cup [s_2,e_2]\cup\cdots\cup[s_L,e_L],\]
    where $s_1 = 1,e_L = T$, $s_{i+1} = e_i + 1$ for all $i\le L-1$, and $\nu_t = \nu_{t'}$ for all $t,t'$ in the same interval. Next, we show that for any execution of the algorithm, the following quantity is upper bounded,
    \[\sum_{t=1}^T\left(\mu^*_t - \mu_{A_t,t}\right)\cdot\I\left\{\cN^s\land \cC_t\land \lnot \cP_t\right\}.\]
    Fix any realization, suppose that $\tau_{\ell}$ denote the starting time of epoch $\ell$ in that realization. Then we can divide the total time horizon $[1,T]$ into the following intervals
    \[[1,T] = [s'_1,e'_1]\cup [s'_2,e'_2]\cup\cdots\cup[s'_{\cS'},e'_{\cS'}],\]
    where in each interval $[s'_i,e'_i]$, the distribution remains the same \emph{and} it is included in an epoch. We can choose $s'_i = s_j$ for some $j$ or $s'_i = \tau_{\ell}$ for some $\ell$. From Lemma \ref{lem:number-of-epochs-switching}, we know that we can have a partition satisfying the previous constraints and $\cS' \le 2\cS$ if $\cN^s$ happens. Otherwise if $\cN^s$ does not happen, the inequality we want to prove holds.
    
    Then we fix $t\in [1,T]$ and suppose that $t\in [s'_i, e'_i]$. Since from the previous argument, we know that for all arm $a < A_t$, $a$ cannot be optimal, so we can only focus on the arms $a\ge A_t$. Since $A_t$ is not eliminated in time $t-1$, we know that for all $a\ge A_t$ and $t > s'_i$, we have
    \[\hat\mu_{a}[s'_i,t) - \hat\mu_{A_t}[s'_i,t) \le \sqrt{\frac{C_1\ln(KT^3)}{t-s'_i}},\]
    and from the definition of $\Elim$, arm $a$ and $A_t$ are observed at any time $s\in [s'_i,t-1]$, because we must play an eliminated arm in an exploration phase. Thus from $\cN^s$ we have
    \[\mu_{a,t} - \mu_{A_t,t} \le (\sqrt{C_1} + \sqrt{2})\cdot \sqrt{\frac{\ln(KT^3)}{t-s'_i}}.\]
    For $t = s'_i$, it is easy to bound $\mu_{a,s'_i} - \mu_{A_t,s'_i} \le 1$. So we have
    \begin{align*}
        &\left(\mu^*_t - \mu_{A_t,t}\right)\cdot\I\left\{\cN^s\land \cC_t\land \lnot \cP_t\right\}
        \le \min\left\{(\sqrt{C_1} + \sqrt{2})\cdot \sqrt{\frac{\ln(KT^3)}{t-s'_i}},1\right\}.
    \end{align*}
    Sum up all $t\in [s'_i,e'_i]$, we get
    \begin{align*}
    \sum_{s'_i}^{e'_i}\left(\mu^*_t - \mu_{A_t,t}\right)\cdot\I\left\{\cN^s\land \cC_t\land \lnot \cP_t\right\}
    \le& \sum_{s'_i}^{e'_i}\min\left\{(\sqrt{C_1} + \sqrt{2})\cdot \sqrt{\frac{\ln(KT^3)}{t-s'_i}},1\right\}\\
    \le& 1 + 2(\sqrt{C_1} + \sqrt{2})\sqrt{\ln(KT^3)}(\sqrt{e'_i - s'_i + 1}),
    \end{align*}
    where we use the fact that
    \begin{align*}
    1 + \frac{1}{\sqrt{2}} +\cdots + \frac{1}{\sqrt{n}}
    \le& 1 + \int_{1}^{n}\frac{1}{\sqrt{x}}dx \\
    =& 1 + 2\sqrt{x}\big|_{1}^{n} \\
    \le& 2\sqrt{n}.
    \end{align*}
    Sum up all intervals, we have
    \begin{align*}
        \sum_{t=1}^T\left(\mu^*_t - \mu_{A_t,t}\right)\cdot\I\left\{\cN^s\land \cC_t\land \lnot \cP_t\right\}
        \le& \cS' + 2(\sqrt{C_1} + \sqrt{2})\sqrt{\ln(KT^3)}\sqrt{\cS'T}\\
        \le& 2\cS + 2(\sqrt{C_1} + \sqrt{2})\sqrt{\ln(KT^3)}\sqrt{2\cS T}.
    \end{align*}
    Then we conclude the proof of this lemma.
\end{proof}

\lemsecondthirdtermswitching*

\begin{proof}
    First, we observe that when both $\lnot \cC_t$ and $\lnot \cB_t$ happen, $t$ is in an exploration phase. This is due to the fact that: From the definition of $\lnot\cB_t$, suppose arm $b_t$ to be the smallest arm such that there exists $(g,e,\bv)\in\cM_{t}$ which satisfies $e = b_t$ and exists $a\ge b_t,|\bv_{a}-\mu_{a,t}|>\frac{g}{4}$, then we have $A_t < b_t$. Since $b_t$ is an eliminated arm, and if $t$ is not in an exploration phase, $A_t > b_t$ from the definition of the algorithm.
    
    Next, we show that either $\cN^s\land \cC_t\land \cP_t$ or $\cN^s\land \lnot \cC_t\land\lnot \cB_t$, we have
    \[\mu^*_t - \mu_{A_t,t} \le 12 d_{\max,t}.\]
    Suppose that  $(g,e,\bv)$ where $e = A_t$ is the tuple corresponds to arm $A_t$ when $A_t$ is eliminated in the current epoch. From the definition of $\cC_t$ and the definition of $\cB_t$, we know that either $\cN^s\land \cC_t\land \cP_t$ or $\cN^s\land \lnot \cC_t\land\lnot \cB_t$, we have
    \[\mu^*_t - \mu_{A_t,t} \le \frac{3}{2} g.\]
    From the definition of algorithm $\Elim$ (Line \ref{algline:exploration-arm}), we know that $g \le 8d_{\max,t}$, which means that
    \[\mu^*_t - \mu_{A_t,t} \le 12 d_{\max,t}.\]
    
    Then, we can observe that $d_{\max,t}$ can have at most $-\log_2 \sqrt{\frac{C_1\ln(KT^3)}{T}}$ different value. This is due to the fact that when a vector $\bv$ is add into the set $\cM$, we know that $g > \sqrt{\frac{C_1\ln(KT^3)}{T}}$, so we have $d_i \ge \frac{1}{8}\sqrt{\frac{C_1\ln(KT^3)}{T}}$. Take the logarithm, we have $d_{\max,t}$ can have at most $3-\log_2 \sqrt{\frac{C_1\ln(KT^3)}{T}}$. For simplicity, we denote $\mathbb A = \big\lfloor 3-\log_2 \sqrt{\frac{C_1\ln(KT^3)}{T}}\big\rfloor$, and $\mathbb A = \cO(\log T)$.
    
    In this way, we can decompose the regret in the following way.
    \begin{align*}
        &\E\left[\sum_{t=1}^T\left(\mu^*_t - \mu_{A_t,t}\right)\cdot\I\left\{\cN^s\land \cC_t\land \cP_t\right\}\right] \\
        &\ +\E\left[\sum_{t=1}^T\left(\mu^*_t - \mu_{A_t,t}\right)\cdot\I\left\{\cN^s\land \lnot \cC_t\land\lnot \cB_t\right\}\right] \\
        \le& \E\left[\sum_{t=1}^T12 d_{\max, t}\cdot \I\left\{\cN^s\land \cP_t\right\}\right] \\
        =& \E\left[\sum_{t=1}^T\sum_{i=0}^{\mathbb A}12 d_{\max, t}\cdot \I\left\{\cN^s\land \cP_t\land d_{\max,t} = 2^{-i}\right\}\right]\\
        =& 12\E\left[\sum_{i=0}^{\mathbb A}\sum_{t=1}^T2^{-i}\cdot \I\left\{\cN^s\land \cP_t\land d_{\max,t} = 2^{-i}\right\}\right]\\
        =& 12\sum_{i=0}^{\mathbb A}\E\left[\sum_{t=1}^T2^{-i}\cdot \I\left\{\cN^s\land \cP_t\land d_{\max,t} = 2^{-i}\right\}\right].
    \end{align*}
    Then we bound $\E\left[\sum_{t=1}^T2^{-i}\cdot \I\left\{\cN^s\land \cP_t\land d_{\max,t} = 2^{-i}\right\}\right]$ for any fixed $i$. We have
    \begin{align*}
        \E\left[\sum_{t=1}^T 2^{-i}\cdot \I\left\{\cN^s\land \cP_t\land d_{\max,t} = 2^{-i}\right\}\right] 
        =& 2^{-i}\sum_{t=1}^T\Pr\left\{\cN^s\land \cP_t\land d_{\max,t} = 2^{-i}\right\} \\
        \le& d_i\sum_{t=1}^T\Pr\left\{\cN^s\land t\in\cI \text{ where } (d_i,\cI)\in\cE\right\} \\
        \le& d_i\sum_{t=1}^T \left(\big\lceil\frac{C_2\ln (KT^3)}{d_i^2}\big\rceil+1\right)d_i\sqrt{\frac{\cS+1}{T}}\\
        \le& C_2\ln (KT^3)\sqrt{(\cS+1)T} + 2\sqrt{\frac{\cS+1}{T}},
    \end{align*}
    where the second inequality comes from the union bound, the definition of algorithm(the probability to add an exploration phase, Line \ref{algline:exploration-prob}), and Lemma \ref{lem:number-of-epochs-switching}. Fininally, we can get
    \begin{align*}
    &\E\left[\sum_{t=1}^T\left(\mu^*_t - \mu_{A_t,t}\right)\cdot\I\left\{\cN^s\land \cC_t\land \cP_t\right\}\right]  +\E\left[\sum_{t=1}^T\left(\mu^*_t - \mu_{A_t,t}\right)\cdot\I\left\{\cN^s\land \lnot \cC_t\land\lnot \cB_t\right\}\right] \\
    \le& \left(C_2\ln (KT^3)\sqrt{(\cS+1)T} + 2\sqrt{\frac{\cS+1}{T}}\right)  \times \left(3-\log_2 \sqrt{\frac{C_1\ln(KT^3)}{T}}\right).
    \end{align*}
\end{proof}

\lemforthtermswitching*

\begin{proof}
    We partition the time interval $[1,T]$ into $\cS$ intervals where the distributions in each interval are identical. We set
    \[[1,T] = [s_1,e_1]\cup [s_2,e_2]\cup\cdots\cup[s_{\cS},e_{\cS}],\]
    where $s_1 = 1,e_{\cS} = T$, $s_{i+1} = e_i + 1$ for all $i\le \cS-1$, and $\nu_t = \nu_{t'}$ for all $t,t'$ in the same interval. For convinience, we will use $\I_t$ to denote the variable $\I\left\{\cN^s\land \lnot \cC_t\land \cB_t\right\}$.
    
    We will use induction to show the lemma. Then we set $R_{j}^{\ell}$ to be the following quantity
    \[\E_{s_{j}}\left[\sum_{t=s_{\ell}}^T\left(\mu^*_t - \mu_{A_t,t}\right)\cdot\I_t\bigg|\tau_{\ell} < s_j \le \tau_{\ell+1}\right],\]
    and $R^{\ell}(s)$ to be the following quantity 
    \[\E_s\left[\sum_{t = s}^T\left(\mu^*_t - \mu_{A_t,t}\right)\cdot\I_t\bigg|\tau_{\ell} = s\right],\]
    where $\E_s$ means that we take the expectation of all the randomness from time $s$ to time $T$. Note that $R_{j}^{\ell}$ and $R^{\ell}(s)$ should be viewed as measurable functions instead of value, since $R_{j}^{\ell}$ and $R^{\ell}(s)$ are conditional expectation given all the information before $s_{j}$ and $s$ and satisfy the constraints. Generally speaking, $R_{j}^{\ell}$ is the conditional filtered regret from $s_{j}$ to $T$ given that $s_{j}$ is in epoch $\ell$, and $R^{\ell}(s)$ is the conditional filtered regret from $s$ to $T$ given that epoch $\ell$ starts at $s$. We only have to show that
    \[R^{1}(1) \le 24\sqrt{(\cS+1)T}+24\sqrt{C_2\ln(KT^3)\cS T} .\]
    We will use induction to show that
    \begin{align*}
        R_{j}^{\ell}\le& 12\sum_{k=\ell}^{\cS}\sqrt{\frac{T}{k+1}}+24\sum_{j'\ge j}\sqrt{C_2\ln(KT^3)(e_{j'}-s_{j'}+1)} . 
    \end{align*}
    The proof is long and we divide them into 3 parts. In the first part we show the induction is correct in the base cases. In the second part we relate $R_{j}^{\ell}$ and $R^{\ell}(s)$. In the third part we finish the induction step using the results in the second part.
    
    The induction has two dimensions $j$ and $\ell$, where $j$ denotes the index of the interval and $\ell$ denote the index of the starting epoch, and the induction proceeds from back to front.
    
    \textbf{Part 1:} We first show that all of the base cases are true. We first show that any $\ell \ge \cS$ and any $j\le \cS$, we have
    \[R_{j}^{\ell} = 0.\]
    From the proof of Lemma \ref{lem:number-of-epochs-switching}, we know that when $\cN^s$ happens, epoch $\cS$ must start after $e_{\cS-1}$. Conditioning on $\tau_{\ell}< s_j \le \tau_{\ell+1}$, for any $\ell \ge \cS$ and any $j\le \cS$, $\cN^s$ will not happen, and $R_{j}^{\ell} = 0$. If we define $s_{\cS+1} = T+1$, and we have $R_{\cS+1}^{\ell} = 0$ for all $\ell$.
    
    \textbf{Part 2:} Same as before, let $\omega_s$ denote a possible realization of all the randomness before time $s$. Next we show that, if for $j_0,\ell_0$, such that for any $\omega_{s_{j_0}}$ satisfies $\tau_{\ell_0} < s_{j_0} \le \tau_{{\ell_0}+1}$, we have
    \begin{align*}
        R_{j_0}^{\ell_0}(\omega_{s_{j_0}})\le& 12\sum_{k=\ell_0}^{\cS}\sqrt{\frac{T}{k+1}}+24\sum_{j'\ge j_0}\sqrt{C_2\ln(KT^3)(e_{j'}-s_{j'}+1)},
    \end{align*}
    then for any $s\in[s_{j_0-1},e_{j_0-1}]$ and any $\omega_s$ satisfied $\tau_{\ell_0} = s$, we have
    \begin{align*}
        R^{\ell_0}(s)(\omega_{s})\le& 12\sum_{k=\ell_0}^{\cS}\sqrt{\frac{T}{k+1}}+24\sum_{j'\ge j_0}\sqrt{C_2\ln(KT^3)(e_{j'}-s_{j'}+1)}.
    \end{align*}
    From the additive property of the conditional expectation, we have
    \begin{align*}
        R^{\ell_0}(s)=& \E_s\left[\sum_{t = s}^T\left(\mu^*_t - \mu_{A_t,t}\right)\cdot\I_t\bigg|\tau_{\ell_0} = s\right]\\
        =& \E_s\left[\underbrace{\sum_{t = s_{j_0}}^{T}\left(\mu^*_t - \mu_{A_t,t}\right)\cdot\I_t}_{\text{term } \mathbb A}\bigg|\tau_{\ell_0} = s\right]\\
        +& \E_s\left[\underbrace{\sum_{t = s}^{s_{j_0-1}}\left(\mu^*_t - \mu_{A_t,t}\right)\cdot\I_t}_{\text{term } \mathbb B}\bigg|\tau_{\ell_0} = s\right]
    \end{align*}
    For the first term, we have
    \begin{align*}
        \I_t = \I\left\{\cN^s\land \lnot \cC_t\land \cB_t\land \cN^s_{s_{j_0-1}}\right\},
    \end{align*}
    because $\cN^s\subseteq\cN^s_{s_{j_0-1}}$. Then from the tower property of the conditional expectation, we have
    \begin{align*}
        \E_s\left[\mathbb A\big|\tau_{\ell_0} = s\right]
        =\E_s\left[\E_{s_{j_0}}\left[\mathbb A\big|\tau_{\ell_0} = s\right]\big|\tau_{\ell_0} = s\right],
    \end{align*}
    and we bound
    \[f(\omega_{s_{j_0}}) = \E_{s_{j_0}}\left[\mathbb A\big|\tau_{\ell_0} = s\right]\]
    for all $\omega_{s_{j_0}}$ satisfies $\tau_{\ell_0} = s$. Note that $\cN^s_{s_{j_0-1}}$ is totally defined given $\omega_{s_{j_0}}$, and we have
    \[f(\omega_{s_{j_0}}) = \E_{s_{j_0}}\left[\mathbb A\big|\tau_{\ell_0} = s\right]\cdot\I\left\{\cN^s_{s_{j_0-1}}\right\}.\]
    If $\cN^s_{s_{j_0-1}}$ happens, we can know from the proof of Lemma \ref{lem:number-of-epochs-switching}(generalize from $\cN^s$ to $\cN^s_{s_{j_0-1}}$) that epoch ${\ell_0}$ will not end in the interval $[s_{j_0-1},e_{j_0-1}]$, and have
    \[\cN^s_{s_{j_0-1}} \land \{\tau_{\ell_0} = s\} \Rightarrow \{\tau_{\ell_0} < s_{j_0} \le \tau_{{\ell_0}+1}\}.\]
    If $\cN^s_{s_{j_0-1}}$ does not happen, then the term inside the expectation is just $0$. Then from our assumption that
    \begin{align*}
        R_{j_0}^{\ell_0}(\omega_{s_{j_0}})\le& 12\sum_{k=\ell_0}^{\cS}\sqrt{\frac{T}{k+1}}+24\sum_{j'\ge j_0}\sqrt{C_2\ln(KT^3)(e_{j'}-s_{j'}+1)},
    \end{align*}
    we have
    \begin{align*}
        \E_{s_{j_0}}\left[\mathbb A\big|\tau_{\ell_0} = s\right]\le 12\sum_{k=\ell_0}^{\cS}\sqrt{\frac{T}{k+1}}
        +24\sum_{j'\ge j_0}\sqrt{C_2\ln(KT^3)(e_{j'}-s_{j'}+1)}.
    \end{align*}
    Then we show that
    \[\E_{s_{j_0}}\left[\mathbb B\big|\tau_{\ell_0} = s\right] = 0.\]
    If $\cN^s$ happens and epoch $\ell$ start at $s\in [s_{j_0-1},e_{j_0-1}]$, then $\I\left\{\cC_t\right\} = 1$ for all $s\le t\le e_{j_0-1}$. This is due to the fact that if we add $\bv$ into $\cM$ at time $t'$ such that $s\le t'\le e_{j_0-1}$, then we know that $\exists a> e$ and $\sigma$,
    \[\hat\mu_a[\sigma,t'] - \hat\mu_{e}[\sigma,t'] > \sqrt{\frac{C_1\ln(KT^3)}{t'+1-\sigma}},\]
    and we have
    \[g = \hat\mu_a[\sigma,t'] - \hat\mu_{e}[\sigma,t'] > \sqrt{\frac{C_1\ln(KT^3)}{t'+1-\sigma}}. \]
    If $\cN^s$ happens, we know that for any arm $b\ge e$, we have
    \begin{equation}\label{equ:no-bad-procedure-in-same-interval}
        |\hat\mu_b[\sigma,t'] - \mu_{b,t'}| < \sqrt{\frac{\ln(KT^3)}{2(t'+1-\sigma)}} < \frac{g}{8},
    \end{equation}
    where the last inequality comes from the fact that we choose $C_1 \ge 32$. Then we finish the proof of \textbf{Part 2}.
    
    \textbf{Part 3:} In this part, we complete the induction step. Suppose that the inequality holds for $R_{j+1}^{\ell}$ and $R_{j+1}^{\ell+1}$, we now show that the inequality also hold for $R_{j}^{\ell}$. First from \textbf{Part 2}, we know that the inequality holds for $R^{\ell+1}(s)$ for all $s\in [s_j,e_j]$.
    
    Given $\omega_{s_j}$, we know that $\cM_{s_j}$ is determined. Then we will show that for any $s_j \le t\le e_j$, we have $\I\left\{\cC_t\land\cN^s\right\} = \I\left\{\cC_{s_j}\land\cN^s\right\}$. We can first assume that $\cN^s$ happens, otherwise both sides of the equation is 0. Then if a mean vector $\bv$ is added into the set $\cM$ by our algorithm during time interval $[s_j,e_j]$, and if we assume that $\cN^s$ happens, then from Equation \ref{equ:no-bad-procedure-in-same-interval}, we know that $\bv$ will not affect whether $\cC_t$ will happen for any $t\in [s_j,e_j]$. In this way, only the vectors $\bv'\in\cM_{s_j}$ affect $\cC_t$. Then, we have $\I\left\{\cC_t\land\cN^s\right\} = \I\left\{\cC_{s_j}\land\cN^s\right\}$ for all $s_j \le t\le e_j$.
    
    Now we can suppose that $\I\left\{\cC_{s_j}\land\cN^s\right\} = 1$, otherwise the induction is absolutely true(The steps $s_j \le t\le e_j$ do not contribute to the filtered regret). Then we assume arm $b_{s_j}$ to be the smallest arm such that $\exists (g,e,\bv)\in\cM_{s_j}$ such that $e = b_{s_j}$ and exists $a\ge b_{s_j},|\bv_{a}-\mu_{a,s_j}|>\frac{g}{4}$. From now on, we use $(g^j,e^j,\bv^j)$ to denote the element such that $e^j = b_{s_j}$. We set $\varepsilon_j = \max_{a\ge b_{s_j}}|\mu_{a,s_j} - \bv^j_{a}|$ to be the max difference from the true mean and the estimated mean. We set $d_{i_j} < \varepsilon_j \le 2d_{i_j}$ where $d_i = 2^{-i}$ is defined in the algorithm. 
    
    Then there are 2 kinds of arms during time $s_j \le t\le e_j$, the first is that $A_t < b_{s_j}$, which only happens in an exploration phase, and the second is that $A_t \ge b_{s_j}$, which may happens both in an exploration phase or not in any exploration phase. However, we can know that $b_t$ are all the same for any $s_j \le t\le e_j$, since if we add $\bv$ into $\cM$ during time interval $[s_j,e_j]$, then $\bv$ will not affect whether $\cC_t$ happens during time interval $[s_j,e_j]$, so $\cB_t = \{A_t \ge b_t\} = \{A_t \ge b_{s_j}\}$.
    
   First, we have
    \[d_{i_j} \ge \frac{\varepsilon_j}{2} \ge \frac{g^j}{8},\]
    and thus for $A_t \ge b_{s_j}$, we have
    \begin{equation}\label{equ:regret-bad-arm}
        \mu_t^* - \mu_{A_t,t} \le 2\varepsilon_j + g^j \le 6\varepsilon_j \le 12 d_{i_j},
    \end{equation}
    since all the arms $a < b_{s_j}$ will not be optimal(from the definition of $b_{s_j}$ and the definition of $\cC_{s_j}$).
    
    Next we decompose the regret $R_{j}^{\ell}(\omega_{s_j})$ into several cases: $\tau_{\ell+1} = s$ for $s\in [s_j,e_j]$ or $\tau_{\ell} < s_{j+1} \le \tau_{\ell+1}$. In the first case, we use $R^{\ell+1}(s)$ to complete the induction and in the second case, we use $R_{j+1}^{\ell}$ to complete the induction. Then we have for any $\omega_{s_j}$ satisfies $\tau_{\ell} < s_j \le \tau_{\ell+1}$.
    \begin{align}
        R_{j}^{\ell}(\omega_{s_j})
        =& \E_{s_{j}}\left[\sum_{t=s_{\ell}}^T\underbrace{\left(\mu^*_t - \mu_{A_t,t}\right)\cdot \I\left\{\cN^s\land \lnot \cC_t\land\cB_t\right\}}_{\text{term }\mathbb C_t}\biggr|\omega_{s_j}\right]\nonumber \\
        =& \sum_{s=s_j}^{e_j}\E_{s_{j}}\left[\sum_{t=s_{\ell}}^T\mathbb C_t\cdot\I\{\tau_{\ell+1} = s\}\biggr|\omega_{s_j}\right]+ \E_{s_{j}}\left[\sum_{t=s_{\ell}}^T\mathbb C_t\cdot\I\{\tau_{\ell} <s_{j+1} \le \tau_{\ell+1} \}\biggr|\omega_{s_j}\right]\nonumber \\
        \le& \sum_{s=s_j}^{e_j}\E_{s_{j}}\left[\sum_{t=s_{\ell}}^{s-1}\mathbb C_t\cdot\I\{\tau_{\ell+1} = s\}\biggr|\omega_{s_j}\right]\label{equ:induction-illus}+ \sum_{s=s_j}^{e_j}\Pr\{\tau_{\ell+1} = s\land \cN^s|\omega_{s_j}\}\cdot ||R^{\ell+1}(s)||_{\infty} \\
        &\ + \E_{s_{j}}\left[\sum_{t=s_{j}}^{e_j}\mathbb C_t\cdot\I\{\tau_{\ell} <s_{j+1} \le \tau_{\ell+1} \}\biggr|\omega_{s_j}\right]+ \Pr\{\tau_{\ell} <s_{j+1}\le \tau_{\ell+1}\land \cN^s|\omega_{s_j}\}\cdot ||R_{j+1}^{\ell}||_{\infty}\nonumber \\
        \le& \sum_{s=s_j}^{e_j}\E_{s_{j}}\left[\sum_{t=s_{\ell}}^{s-1}\mathbb C_t\cdot\I\{\tau_{\ell+1} = s\}\biggr|\omega_{s_j}\right] + \E_{s_{j}}\left[\sum_{t=s_{j}}^{e_j}\mathbb C_t\cdot\I\{\tau_{\ell} <s_{j+1} \le \tau_{\ell+1} \}\biggr|\omega_{s_j}\right]\nonumber  \\
        &\ + \Pr\{\tau_{\ell} <s_{j+1}\le \tau_{\ell+1}\land \cN^s|\omega_{s_j}\}12\sqrt{\frac{T}{\ell+1}}\nonumber \\
        &\ +12\sum_{k=\ell+1}^{\cS}\sqrt{\frac{T}{k+1}}+24\sum_{j'\ge j+1}\sqrt{C_2\ln(KT^3)(e_{j'}-s_{j'}+1)},\nonumber 
    \end{align}
    where the Equation (\ref{equ:induction-illus}) comes from the fact that for any $\omega$ such that $\lnot\cN^s$ happens, $R^{\ell+1}(s)(\omega) = R^{\ell}_{j+1}(\omega) = 0$, and the last 3 terms comes from the induction step. Then to complete the induction step, we just have to show that
    \begin{align}
        & \sum_{s=s_j}^{e_j}\E_{s_{j}}\left[\sum_{t=s_{\ell}}^{s-1}\mathbb C_t\cdot\I\{\tau_{\ell+1} = s\}\biggr|\omega_{s_j}\right]\nonumber\\
        &\ + \E_{s_{j}}\left[\sum_{t=s_{j}}^{e_j}\mathbb C_t\cdot\I\{\tau_{\ell} <s_{j+1} \le \tau_{\ell+1} \}\biggr|\omega_{s_j}\right]\nonumber \\
        &\ + \Pr\{\tau_{\ell} <s_{j+1}\le \tau_{\ell+1}\land \cN^s|\omega_{s_j}\}12\sqrt{\frac{T}{\ell+1}}\nonumber\\
        \le& 12\sqrt{\frac{T}{\ell+1}} + 24\sqrt{C_2\ln(KT^3)(e_{j}-s_{j}+1)}\label{equ:induction}.
    \end{align}
    We first observe that from Equation \ref{equ:regret-bad-arm},
    \begin{align*}
        \mathbb C_t =& \left(\mu^*_t - \mu_{A_t,t}\right)\cdot \I\left\{\cN^s\land \lnot \cC_t\land\cB_t\right\}\\
        \le& 12 d_{i_j}\cdot \I\left\{\cN^s\land \lnot \cC_t\land\cB_t\right\}.
    \end{align*}
    Then if the length of interval $[s_j,e_j]$ satisfies
    \[s_j-e_j + 1 \le 2\left\lceil\frac{C_2\ln (KT^3)}{d_{i_j}^2}\right\rceil,\]
    then when $KT^3>1$, we have
    \begin{align*}
        d_{i_j} \le \sqrt{\frac{4C_2\ln (KT^3)}{s_j-e_j+1}},
    \end{align*}
    and thus we can show that
    \begin{align*}
         & \sum_{s=s_j}^{e_j}\E_{s_{j}}\left[\sum_{t=s_{\ell}}^{s-1}\mathbb C_t\cdot\I\{\tau_{\ell+1} = s\}\biggr|\omega_{s_j}\right]\\
        &\ + \E_{s_{j}}\left[\sum_{t=s_{j}}^{e_j}\mathbb C_t\cdot\I\{\tau_{\ell} <s_{j+1} \le \tau_{\ell+1} \}\biggr|\omega_{s_j}\right] \\
        \le& 12d_{i_j}(s_j-e_j+1)\\
        \le& 24\sqrt{C_2\ln (KT^3)(s_j-e_j+1)},
    \end{align*}
    and we finish the induction step in this case using the fact that
    \[\Pr\{\tau_{\ell} <s_{j+1}\le \tau_{\ell+1}\land \cN^s|\omega_{s_j}\} \le 1.\]
    Then we just have to finish the induction step in the second case where
    \[s_j-e_j + 1 > 2\left\lceil\frac{C_2\ln (KT^3)}{d_{i_j}^2}\right\rceil \ge \frac{2C_2\ln (KT^3)}{d_{i_j}^2}. \]
    Then the key observation is that, when $\cN^s$ happens and an exploration phase $\left(d_{i_j},\left[t,t+\lceil\frac{C_2\ln (KT^3)}{d_{i_j}^2}\rceil\right)\right)$ is inserted into $\cM$ in the time interval $t\in \left[s_j,e_j-\left\lceil\frac{C_2\ln (KT^3)}{d_{i_j}^2}\right\rceil+1\right]$, the non-stationary test will detect the non-stationary at $t' = t+\lceil\frac{C_2\ln (KT^3)}{d_{i_j}^2}\rceil-1$ and the new epoch will start at time $t'+1$. This is due to the fact that: As $d_{i_j} \ge \frac{\varepsilon_j}{2} \ge \frac{g^j}{8}$ and $(g^j,e^j,\bv^j)$ is in $\cM$ at time $s_j$, then from Lemma \ref{lem:exploration-observation}, we know that all the arms $a\ge b_{s_j}$ are observed for $\lceil\frac{C_2\ln (KT^3)}{d_{i_j}^2}\rceil$ times. Then for the arm $a$ such that $|\mu_{a,s_j} -\bv^j_{a}| = \varepsilon_j > d_{i_j}$, we know that by $\cN^s$, we have
    \begin{equation}\label{equ:key-observation-switching}
        |\mu_{s_j} -\hat\mu_a\left[t,t'+1\right)| \le \frac{d_{i_j}}{\sqrt{2C_2}} \le \frac{d_{i_j}}{8},
    \end{equation}
    and thus we can conclude
    \[|\bv^j_{a} - \hat\mu_a\left[t,t'+1\right)| > \frac{7d_{i_j}}{8} \ge \frac{d_{i_j}}{4},\]
    and the epoch will end at time $t'$. Then from the key observation before, we have
    \[\Pr\{\tau_{\ell} <s_{j+1}\le \tau_{\ell+1}\land \cN^s|\omega_{s_j}\} \le (1-p_{\ell,i_j})^{e_j+1-s_j-w_{i_j}},\]
    where $p_{\ell,i} = d_i\sqrt{\frac{\ell+1}{T}}$ is the probability to add a sampling phase with index $i$ at round $\ell$(see Algorithm \ref{alg:nonstationaryelim}) and $w_{i} = \left\lceil\frac{C_2\ln (KT^3)}{d_{i}^2}\right\rceil$ is a shorthand of the length of the sampling phase with index $i$.
    
    Then we bound
    \begin{align}
        & \sum_{s=s_j}^{e_j}\E_{s_{j}}\left[\sum_{t=s_{\ell}}^{s-1}\mathbb C_t\cdot\I\{\tau_{\ell+1} = s\}\biggr|\omega_{s_j}\right]+ \E_{s_{j}}\left[\sum_{t=s_{j}}^{e_j}\mathbb C_t\cdot\I\{\tau_{\ell} <s_{j+1} \le \tau_{\ell+1} \}\biggr|\omega_{s_j}\right]\label{equ:left-term}.
    \end{align}
    As $\mathbb C_t = \left(\mu^*_t - \mu_{A_t,t}\right)\cdot \I\left\{\cN^s\land \lnot \cC_t\land\cB_t\right\}$ and when $\cB_t$ happens, from Equation (\ref{equ:regret-bad-arm}), we have $\mu_t^* - \mu_{A_t,t} \le 12d_{i_j}$. Then we can bound (\ref{equ:left-term}) by
    \begin{align*}
    & \sum_{s=s_j}^{e_j}\E_{s_{j}}\left[\sum_{t=s_{\ell}}^{s-1}12d_{i_j}\cdot\I\{\tau_{\ell+1} = s\land\cN^s\}\biggr|\omega_{s_j}\right]\\
    &\ + \E_{s_{j}}\left[\sum_{t=s_{j}}^{e_j}12d_{i_j}\cdot\I\{\tau_{\ell} <s_{j+1} \le \tau_{\ell+1} \land\cN^s\}\biggr|\omega_{s_j}\right]\\
    =& 12d_{i_j}\sum_{s=s_j}^{e_j}\Pr\left\{\tau_{\ell+1}> s\land\cN^s\big|\omega_{s_j}\right\}.
    \end{align*}
    For $s\le s_j + w_{i_j}-1$, we bound $\Pr\left\{\tau_{\ell+1}> s\land\cN^s\big|\omega_{s_j}\right\} \le 1$, and for $s\ge s_j + w_{i_j}$, we bound 
    \[\Pr\left\{\tau_{\ell+1}> s\land\cN^s\big|\omega_{s_j}\right\} \le (1-p_{\ell,i_j})^{s+1-s_j-w_{i_j}}.\]
    Then we have
    \begin{align*}
        12d_{i_j}\sum_{s=s_j}^{e_j}\Pr\left\{\tau_{\ell+1}> s\land\cN^s\big|\omega_{s_j}\right\}
        \le& 12d_{i_j}w_{i_j} + 12d_{i_j}\sum_{r=1}^{e_j+1-s_j-w_{i_j}}(1-p_{\ell,i_j})^{r}\\
        \le& 12d_{i_j}\left[w_{i_j} + \frac{1-(1-p_{\ell,i_j})^{e_j+1-s_j-w_{i_j}}}{p_{\ell,i_j}}\right]\\
        =& 12d_{i_j}w_{i_j} + 12\frac{d_{i_j}}{p_{\ell,i_j}}\left[1-(1-p_{\ell,i_j})^{e_j+1-s_j-w_{i_j}}\right]\\
        \le& 12\sqrt{\frac{T}{\ell+1}}\left[1-(1-p_{\ell,i_j})^{e_j+1-s_j-w_{i_j}}\right] + 24\frac{C_2\ln(KT^3)}{d_{i_j}}\\
        \le& 12\sqrt{\frac{T}{\ell+1}}\left[1-(1-p_{\ell,i_j})^{e_j+1-s_j-w_{i_j}}\right] + 24\sqrt{C_2\ln (KT^3)(s_j-e_j+1)}.
    \end{align*}
    Then we finish the induction(Equation \ref{equ:induction}) by the previous argument that
    \[\Pr\{\tau_{\ell} <s_{j+1}\le \tau_{\ell+1}\land \cN^s|\omega_{s_j}\} \le (1-p_{\ell,i_j})^{e_j+1-s_j-w_{i_j}}.\]
    
    Given the induction result, we conclude the proof by
    \begin{align*}
    R^{1}(1)\le& 12\sum_{k=1}^{\cS}\sqrt{\frac{T}{k+1}}+24\sum_{j'\ge 1}\sqrt{C_2\ln(KT^3)(e_{j'}-s_{j'}+1)}\\
    \le& 24\sqrt{(\cS+1)T}+24\sqrt{C_2\ln(KT^3)\cS T}.
    \end{align*}
\end{proof}

As shown before, the next lemma is a simple application of the high probability results of $\cN^s$.

\lemfifthswitching*

\begin{proof}
    From Lemma \ref{lem:samplingnice}, we know that $\Pr\{\lnot\cN^s\} \le \frac{2}{T}$. Then, since the distribution $\nu_t$ has support on $[0,1]^K$, we know that $|\mu_{A_t,t} - \mu^*_t| \le 1$, and we have
    \begin{align*}
        \E\left[\sum_{t=1}^T\left(\mu^*_t - \mu_{A_t,t}\right)\cdot\I\left\{\lnot\cN^s\right\}\right]
        \le& \E\left[\sum_{t=1}^T\I\left\{\lnot\cN^s\right\}\right] \\
        =& T\cdot \E\left[\I\left\{\lnot\cN^s\right\}\right] \\
        =& T\cdot \Pr\{\lnot\cN^s\}\\
        \le& 2.
    \end{align*}
\end{proof}

Combining the previous lemma, we have the following theorem for the one-sided full-information bandit in the switching case.

\thmswitchingregret*

\begin{proof}[Proof of Theorem \ref{thm:switching-regret}]
    The proof of Theorem \ref{thm:switching-regret} is a direct combination of Lemma \ref{lem:decomposing-switching},\ref{lem:first-term-switching},\ref{lem:secondthird-term-switching},\ref{lem:forth-term-switching} and \ref{lem:fifth-term-switching}.
\end{proof}

\subsection{Dynamic Regret}

In this section, we give the proof of Theorem \ref{thm:dynamic-regret}. The proof is an imitation of the proof strategy of Theorem \ref{thm:switching-regret}. Although the mean can be changing at every time $t\in [1,T]$, we can approximatly divide them into several subinterval such that in each interval, the change of mean is not large. Recall that $\cV_{\cI} := \sum_{t=s+1}^{s'}\max_{a}||\mu_{a,t} -\mu_{a,t-1}||$ and we use $\cV := \cV_{[1,T]}$. We have the following lemma,

\lemintervalpartition*

The construction of the partition can be implemented by a simple greedy algorithm(Algorithm \ref{alg:interval-partition}), and the proof can be concluded by a simple application of the H\"older's Inequality. The proof is very similar to the proof of Lemma 5 in \cite{chen2019anewalgorithm}. For self-completeness, we include the proof from \cite{chen2019anewalgorithm} and do some small revision.

\begin{algorithm*}[!ht]
    \caption{Interval Partition}
    \label{alg:interval-partition}
    \begin{algorithmic}[1]
        \Require Time horizon $T$.
        \Ensure Disjoint intervals $\cI_1\cup\cdots\cup\cI_{\Gamma} = [1,T]$.
        \State $k\leftarrow 1, s_1 \leftarrow 1, t = 1$
        \While{$t\le T$}
            \If{$\cV_{[s_k,t]} \le \sqrt{\frac{C_3}{t-s_k+1}}$ and $\cV_{[s_k,t+1]} > \sqrt{\frac{C_3}{t-s_k+2}}$}
                \State $e_k \leftarrow t,\cI_k \leftarrow [s_k,e_k]$,
                \State $k \leftarrow k+1, s_k \leftarrow t+1$.
            \EndIf
            \State $t\leftarrow t+1$.
        \EndWhile
        \If{$s_k \le T$}
            \State $e_k \leftarrow T,\cI_k \leftarrow [s_k,e_k]$
        \EndIf
    \end{algorithmic}
\end{algorithm*}

\begin{proof}
    We only have to show that, by the construction shown in Algorithm \ref{alg:interval-partition}, the number of intervals $\Gamma$ satisfies $\Gamma \le \left(\frac{2T}{C_3}\right)^{1/3}\cV^{\frac{2}{3}}+ 1$.
    
    If the retuning intervals satisfy $\Gamma > 1$, we have
    \begin{align*}
        \cV \ge& \cV_{[s_1,e_1+1]} + \cV_{[s_2,e_k+2]} + \cdots + \cV_{[s_{\Gamma-1},e_{\Gamma-1}+1]} \\
            \ge& \sum_{k=1}^{\Gamma-1}\sqrt{\frac{C_3}{e_k - s_k + 1}} \\
            =& \sum_{k=1}^{\Gamma-1}\sqrt{\frac{C_3}{|\cI_k| + 1}}.
    \end{align*}
    Then applying the H\"older's Inequality, we have
    \begin{align*}
        &\left(\sum_{k=1}^{\Gamma-1}\sqrt{\frac{C_3}{|\cI_k| + 1}}\right)^{\frac{2}{3}}\left(\sum_{k=1}^{\Gamma-1}(|\cI_k| + 1)\right)^{\frac{1}{3}} \\
        \ge& \sum_{k=1}^{\Gamma - 1}\left(\sqrt{\frac{C_3}{|\cI_k| + 1}}\right)^{\frac{2}{3}}\left(|\cI_k| + 1\right)^{\frac{1}{3}} \\
        =& (\Gamma - 1)C_3^{1/3}.
    \end{align*}
    On the other hand, we have
    \begin{align*}
        &\left(\sum_{k=1}^{\Gamma-1}\sqrt{\frac{C_3}{|\cI_k| + 1}}\right)^{\frac{2}{3}}\left(\sum_{k=1}^{\Gamma-1}(|\cI_k| + 1)\right)^{\frac{1}{3}} \\
        \le& \cV^{\frac{2}{3}} \left(\sum_{k=1}^{\Gamma-1}2|\cI_k|\right)^{\frac{1}{3}} \\
        \le& \cV^{\frac{2}{3}}(2T)^{\frac{1}{3}}.
    \end{align*}
    Then we have
    \[\Gamma \le \left(\frac{2T}{C_3}\right)^{1/3}\cV^{\frac{2}{3}}+ 1.\]
\end{proof}

Note that this partition of interval only happens in the analysis of the algorithm, but not the implementation of the algorithm. Then given the previous interval partition lemma, we can now mimic the proof of Theorem \ref{thm:switching-regret}. 

From now on, we will choose the parameters $C_1 \ge 8192, C_2 \ge 128, C_3 \le \frac{1}{2}$, and we fix a partition of the interval $[1,T]$ into $\cI_1\cup\cI_2\cup\cdots\cup\cI_{\Gamma}$ such that $\cI_i \cap \cI_j = \phi$, and for any $i\le\Gamma$, $\cV_{\cI_i} \le \sqrt{\frac{C_3}{|\cI_i|}}$, and $\Gamma \le \left(\frac{2T}{C_3}\right)^{1/3}\cV^{\frac{2}{3}}+ 1$. We also use the notations $\cI_j = [s_j,e_j]$ for convenience.

First, we show that given the partition, the difference of the average mean of 2 sub-index set cannot differ a lot.

\begin{restatable}{lemma}{lemmeandifference}\label{lem:mean-difference}
    Suppose that we partition the interval $[1,T]$ into $\cI_1\cup\cI_2\cup\cdots\cup\cI_{\Gamma}$ such that $\cI_i \cap \cI_j = \phi$, and for any $i\le\Gamma$, $\cV_{\cI_i} \le \sqrt{\frac{C_3}{|\cI_i|}}$. Then for any interval $\cI_i$ and any index sets $\cJ_1,\cJ_2 \subseteq \cI_i$, any arm $a$, we have
    \[\biggr|\frac{1}{|\cJ_1|}\sum_{s\in\cJ_1}\mu_{a,s} - \frac{1}{|\cJ_2|}\sum_{s\in\cJ_2}\mu_{a,s}\biggr| \le \sqrt{\frac{C_3}{|\cI_i|}}.\]
\end{restatable}
    
\begin{proof}
    First for any $t\in\cI_i$, we have the following
    \begin{align*}
        \biggr|\mu_{a,t} - \frac{1}{|\cJ|}\sum_{s\in\cJ}\mu_{a,s}\biggr| 
        =& \biggr| \frac{1}{|\cJ|}\sum_{s\in\cJ}(\mu_{a,t} - \mu_{a,s})\biggr| \\
        \le& \frac{1}{|\cJ|}\sum_{s\in\cJ}|\mu_{a,t} - \mu_{a,s}| \\
        \le& \frac{1}{|\cJ|}\sum_{s\in\cJ}\cV_{\cI_i} \\
        \le& \sqrt{\frac{C_3}{|\cI_i|}}.
    \end{align*}
    Then, we have
    \begin{align*}
        \biggr|\frac{1}{|\cJ_1|}\sum_{s\in\cJ_1}\mu_{a,s} - \frac{1}{|\cJ_2|}\sum_{s\in\cJ_2}\mu_{a,s}\biggr|
        =& \biggr|\frac{1}{|\cJ_1|}\sum_{s'\in\cJ_1}\left(\mu_{a,s'} - \frac{1}{|\cJ_2|}\sum_{s\in\cJ_2}\mu_{a,s}\right)\biggr| \\
        \le& \frac{1}{|\cJ_1|}\sum_{s'\in\cJ_1}\biggr|\mu_{a,s'} - \frac{1}{|\cJ_2|}\sum_{s\in\cJ_2}\mu_{a,s}\biggr| \\
        \le& \frac{1}{|\cJ_1|}\sum_{s'\in\cJ_1}\sqrt{\frac{C_3}{|\cI_i|}} \\
        =& \sqrt{\frac{C_3}{|\cI_i|}}.
    \end{align*}
\end{proof}

Then, we have the corresponding lemma of Lemma \ref{lem:number-of-epochs-switching} showing that the number of epochs in the dynamic case is bounded. The proof is also an imitation, except that we have to consider the change of mean in each sub-interval.

\begin{restatable}{lemma}{lemnumberofepochsdynamic}\label{lem:number-of-epochs-dynamic}
    Suppose that we partition the interval $[1,T]$ into $\cI_1\cup\cI_2\cup\cdots\cup\cI_{\Gamma}$ such that $\cI_i \cap \cI_j = \phi$, and for any $i\le\Gamma$, $\cV_{\cI_i} \le \sqrt{\frac{C_3}{|\cI_i|}}$. When $\cN^s$ happens, we have at time $T$, $\ell \le \Gamma$, i.e. the number of epochs will not exceed the number of partitioned intervals.
\end{restatable}

\begin{proof}
    Given the partition $[1,T] =\cI_1\cup\cI_2\cup\cdots\cup\cI_{\Gamma}$, where $\cV_{\cI_i} \le \sqrt{\frac{C_3}{|\cI_i|}}$ for all $i\le \Gamma$, we only have to show that, if $\cN^s$ happens and epoch $\ell$ starts at time $t$ in interval $\cI_i = [s_i,e_i]$, epoch $\ell$ will not end in the interval $[s_i,e_i]$.
    
    Like Lemma \ref{lem:number-of-epochs-switching}, we prove by contradiction, $\tau_{\ell},\tau_{\ell+1}\in [s_i,e_i]$. Since epoch $\ell$ ends in time interval $[s_i,e_i]$, we know that from the definition of algorithm(Step 4), $\exists t\in [s_i,e_i], (d,[t',t+1))\in\cE_t, (g,e,\bv)\in\cM'_t,a\ge e$ such that $v_1 \le 8d$ and $|\hat\mu_a[t',t+1) - \bv_{a}| > \frac{d}{4}$, where $\cM'_t$ denote the set $\cM$ in time $t$ just before Step 4. Moreover, $\tau_{\ell+1} = t+1$. From now on, we will fix the variables  $d,a,t,\bv$. However we will show that when $\cN^s$ happens, $|\hat\mu_a[t',t+1) - \bv_{a}| \le \frac{d}{4}$.
    
    Like before, we have $\min\{e:(g,e,\bv)\in\cM_t,g\le 8d\} = \min\{e:(g,e,\bv)\in\cM'_t,g\le 8d\}$.
    
    Then from Lemma \ref{lem:exploration-observation}, we know that arm $a$ has been observed for $t-t'+1 = \lceil\frac{C_2\ln (KT^3)}{d^2}\rceil$ times in the interval $[t',t]$, because $a \ge e \ge \min\{e^*:(g^*,e^*,\bv^*)\in\cM'_t,g^*\le 8d\} = \min\{e^*:(g^*,e^*,\bv^*)\in\cM_t,g^*\le 8d\}$, Lemma \ref{lem:mean-difference}, and from the definition of $\cN^s$(Definition \ref{def:sampling-nice}), we have
    \begin{align*}
    |\hat\mu_a[t',t+1)-\mu_{a,t}|
    =& \biggr|\hat\mu_a[t',t+1)-\frac{1}{t+1-t'}\sum_{s=t'}^t\mu_{a,s}+\frac{1}{t+1-t'}\sum_{s=t'}^t\mu_{a,s}-\mu_{a,t}\biggr|\\
    \le&\biggr|\hat\mu_a[t',t+1)-\frac{1}{t+1-t'}\sum_{s=t'}^t\mu_{a,s}\biggr|+\biggr|\frac{1}{t+1-t'}\sum_{s=t'}^t\mu_{a,s}-\mu_{a,t}\biggr|\\
    \le& \sqrt{\frac{\ln(KT^3)}{2(t-t'+1)}} + \sqrt{\frac{C_3}{|\cI_i|}}\\
    \le& 2\sqrt{\frac{\ln(KT^3)}{2\lceil\frac{C_2\ln (KT^3)}{d^2}\rceil}}\\
    \le& 2\sqrt{\frac{\ln(KT^3)d^2}{2C_2\ln (KT^3)}}\\
    \le& \frac{2d}{\sqrt{2C_2}},
    \end{align*}
    where we use the assumption that $C_3 \le \frac{1}{2}$ and $\tau_{\ell+1} \in\cI_i$.
    
    Then we consider $\bv_{a}$. Suppose that the element $(g,e,\bv)$ is added into the set $\cM$ at time $s$, and we know that there exists $\sigma$ such that
    \[g > \sqrt{\frac{C_1\ln(KT^3)}{s+1-\sigma}},\]
    and from $\cN^s$, we have
    \begin{align*}
    |\bv_{a} - \mu_{a,s}| =& |\hat\mu_{a}[\sigma,s+1) - \mu_{a,s}| \\
    =& \biggr|\hat\mu_a[\sigma,s+1)-\frac{1}{s+1-\sigma}\sum_{t=\sigma}^s\mu_{a,t}+\frac{1}{s+1-\sigma}\sum_{t=\sigma}^s\mu_{a,t}-\mu_{a,s}\biggr|\\
    \le&\biggr|\hat\mu_a[\sigma,s+1)-\frac{1}{s+1-\sigma}\sum_{t=\sigma}^s\mu_{a,t}\biggr| +\biggr|\frac{1}{s+1-\sigma}\sum_{t=\sigma}^s\mu_{a,t}-\mu_{a,s}\biggr|\\
    < & \sqrt{\frac{\ln(KT^3)}{2(s-\sigma+1)}}+\sqrt{\frac{C_3}{|\cI_i|}}\\
    < & \frac{2}{\sqrt{2C_1}}g \\
    \le & \frac{16d}{\sqrt{2C_1}}.
    \end{align*}
    Then from the choice of parameters $C_1,C_2$ that $C_1 \ge 8192,C_2 \ge 128$, and $s,t\in [s_i,e_i]$, we have $\mu_{a,s} = \mu_{a,t}$ and thus
    \begin{align*}
    |\hat\mu_a[t',t+1) - \bv_{a}|
    \le& |\hat\mu_a[t',t+1)-\mu_{a,t}| + |\bv_{a} - \mu_{a,s}|\\
    <& \frac{2d}{\sqrt{2C_2}} + \frac{16d}{\sqrt{2C_1}} \\
    \le& \frac{d}{8} + \frac{d}{8} \\
    =& \frac{d}{4}.
    \end{align*}
    Then we conclude the proof of this lemma.
\end{proof}

Because we fix a partition of the interval $[1,T]$ into $\cI_1\cup\cI_2\cup\cdots\cup\cI_{\Gamma}$ such that $\cI_i \cap \cI_j = \phi$, and for any $i\le\Gamma$, $\cV_{\cI_i} \le \sqrt{\frac{C_3}{|\cI_i|}}$, and $\Gamma \le \left(\frac{2T}{C_3}\right)^{1/3}\cV^{\frac{2}{3}}+ 1$, we define
\[\bar\mu_{a,t} = \frac{1}{|\cI_j|}\sum_{s\in\cI_j}\mu_{a,s},\]
where $t\in\cI_j$. We also define $\bar\mu^*_t := \max_a\bar\mu_{a,t}$. Then considering $\bar\mu_{a,t}$ for all arms $a$ and time $t\le T$, it becomes ``an instance in switching case'' with number of switches $\Gamma$. We will formalize this idea in the following proofs.

We have the corresponding definitions in the dynamic case.

\begin{restatable}[Records are Consistent(Dynamic)]{definition}{defprocedurenicedynamic}\label{def:procedure-nice-dynamic}
    For a fixed partition of the interval $[1,T]$ into $\cI_1\cup\cI_2\cup\cdots\cup\cI_{\Gamma}$ such that $\cI_i \cap \cI_j = \phi$, and for any $i\le\Gamma$, $\cV_{\cI_i} \le \sqrt{\frac{C_3}{|\cI_i|}}$ and the definition of $\hat\mu_{a,t}$ before, we say that the procedure is nice at time $t$ if for all $(g,e,\bv)\in\cM_t$ such that for any arm $a \ge e$,
    \[|\bar\mu_{a,t}-\bv_{a}| \le \frac{g}{4}.\]
    We use $\cC^{D}_t$ to denote this event.
\end{restatable}

\begin{restatable}[Playing bad arm(Dynamic)]{definition}{defplayingbadarmdynamic}\label{def:playing-bad-arm-dynamic}
    Let $b_t$ denote the smallest arm such that $\exists (g,e,\bv)\in\cM_{t}$ such that $e = b_t$ and exists $a\ge b_t,|\bv_{a}-\bar\mu_{a,t}|>\frac{g}{4}$, i.e.
    \[b_t = \min\left\{e:(g,e,\bv)\in\cM_t,\exists a\ge b_t,|\bv_{a}-\bar\mu_{a,t}|>\frac{g}{4}\right\}.\]
    We use $\cB^D_t$ to denote the event $\{A_t \ge b_t\}$.
\end{restatable}

Then given these definitions, we have the following lemma for decomposing the regret.

\begin{restatable}[Dynamic Regret Decomposition]{lemma}{lemdecomposingdynamic}\label{lem:decomposing-dynamic}
    The regret can be bounded into the following way,
    \begin{align*}
    \E\left[\sum_{t=1}^T\left(\mu^*_t - \mu_{A_t,t}\right)\right]
    \le& \E\left[\sum_{t=1}^T\left(\bar\mu^*_t - \bar\mu_{A_t,t}\right)\cdot\I\left\{\cN^s\land \cC^{D}_t\land \lnot \cP_t\right\}\right]\\
    &\quad+ \E\left[\sum_{t=1}^T\left(\bar\mu^*_t - \bar\mu_{A_t,t}\right)\cdot\I\left\{\cN^s\land \cC^{D}_t\land \cP_t\right\}\right]\\
    &\quad +\E\left[\sum_{t=1}^T\left(\bar\mu^*_t - \bar\mu_{A_t,t}\right)\cdot\I\left\{\cN^s\land \lnot \cC^{D}_t\land\lnot \cB^D_t \right\}\right]\\
    &\quad +\E\left[\sum_{t=1}^T\left(\bar\mu^*_t - \bar\mu_{A_t,t}\right)\cdot\I\left\{\cN^s\land \lnot \cC^{D}_t\land \cB^D_t\right\}\right]\\
    &\quad +\E\left[\sum_{t=1}^T\left(\bar\mu^*_t - \bar\mu_{A_t,t}\right)\cdot\I\left\{\lnot\cN^s\right\}\right].\\
    &\quad +2\sum_{j=1}^{\Gamma}|\cI_j|\cdot\cV_{\cI_j}.
    \end{align*}
\end{restatable}

\begin{proof}
    The first 5 terms add up to
    \[\E\left[\sum_{t=1}^T\left(\bar\mu^*_t - \bar\mu_{A_t,t}\right)\right].\]
    Then, we have
    \begin{align*}
        \E\left[\sum_{t=1}^T\left(\mu^*_t - \mu_{A_t,t}\right)\right] - \E\left[\sum_{t=1}^T\left(\bar\mu^*_t - \bar\mu_{A_t,t}\right)\right] 
        \le& \E\sum_{t=1}^T|\mu^*_t - \bar\mu^*_t| + \E\sum_{t=1}^T|\mu_{A_t,t} - \bar\mu_{A_t,t}| \\
        \le& 2\sum_{j=1}^{\Gamma}|\cI_j|\cdot\cV_{\cI_j},
    \end{align*}
    where the last inequality derives from Lemma \ref{lem:mean-difference}.
\end{proof}

Then we bound the first five terms in the following four lemmas. The lemmas correspond to those in the switching case.

\begin{restatable}{lemma}{lemfirsttermdynamic}\label{lem:first-term-dynamic}
    \begin{align*}
    &\E\left[\sum_{t=1}^T\left(\bar\mu^*_t - \bar\mu_{A_t,t}\right)\cdot\I\left\{\cN^s\land \cC^{D}_t\land \lnot \cP_t\right\}\right]\\
    \le& 2\Gamma + 2(\sqrt{C_1} + \sqrt{2})\sqrt{\ln(KT^3)}\sqrt{2\Gamma T} + 2\sum_{j=1}^{\Gamma}|\cI_j|\cV_{\cI_j}.
    \end{align*}
\end{restatable}

\begin{proof}
    Like the proof of Lemma \ref{lem:first-term-switching}, the first observation is that, when $\cC^{D}_t$ happens, all the arms $a < A_t$ cannot be the optimal arm in terms of the means $\bar\mu_{a,t}$. The observation is based on the fact that: All the arms $a < A_t$ are eliminated in the current epoch, and suppose that event $\cC^{D}_t$ happens. For any eliminated arm $a$ and its corresponding vector $\bv(e = a)$, we know that there exists arm $b$ such that $\bv_{b} - \bv_{a} = g$, and we have
    \begin{align*}
    \bar\mu_{b,t} - \bar\mu_{a,t} \ge& \bv_{b} - \bv_{a} - |\bv_{b} - \bar\mu_b| - |\bv_{a} - \bar\mu_a| \\
    \ge& g - 2\times \frac{g}{4} \\
    \ge& \frac{g}{2}.
    \end{align*}
    Fix any realization, suppose that $\tau_{\ell}$ denote the starting time of epoch $\ell$ in that realization. Then we can divide the total time horizon $[1,T]$ into the following intervals
    \[[1,T] = [s'_1,e'_1]\cup [s'_2,e'_2]\cup\cdots\cup[s'_{\Gamma'},e'_{\Gamma'}],\]
    where in each interval $[s'_i,e'_i]$ is included in an interval $\cI_j$ for some $j$ \emph{and} it is included in an epoch. We can choose $s'_i = s_j$ for some $j$ or $s'_i = \tau_{\ell}$ for some $\ell$. From Lemma \ref{lem:number-of-epochs-dynamic}, we know that we can have a partition satisfying the previous constraints and $\Gamma' \le 2\Gamma$ if $\cN^s$ happens. Otherwise if $\cN^s$ does not happen, the inequality we want to prove holds.
    
    Then we fix $t\in [1,T]$ and suppose that $t\in [s'_i, e'_i]$. Since from the previous argument, we know that for all arm $a < A_t$, $a$ cannot have the optimal average mean($\bar\mu_{a,t}$), so we can only focus on the arms $a\ge A_t$. Since $A_t$ is not eliminated in time $t-1$, we know that for all $a\ge A_t$ and $t > s'_i$, we have
    \[\hat\mu_{a}[s'_i,t) - \hat\mu_{A_t}[s'_i,t) \le \sqrt{\frac{C_1\ln(KT^3)}{t-s'_i}},\]
    and from the definition of $\Elim$, arm $a$ and $A_t$ are observed at any time $s\in [s'_i,t-1]$, because we must play an eliminated arm in an exploration phase. Thus from $\cN^s$ we have
    \begin{align*}
        \bar\mu_{a,t} - \bar\mu_{A_t,t}
        \le& |\hat\mu_{a}[s'_i,t) - \hat\mu_{A_t}[s'_i,t)|\\
        &\ +\biggr|\frac{\sum_{s=1}^{t-1}\mu_{a,s}}{t-s'_i} - \hat\mu_{a}[s'_i,t)\biggr| + \biggr|\frac{\sum_{s=1}^{t-1}\mu_{A_t,s}}{t-s'_i} - \hat\mu_{A_t}[s'_i,t)\biggr| \\
        &\ +\biggr|\bar\mu_{a,t} - \frac{\sum_{s=1}^{t-1}\mu_{a,s}}{t-s'_i}\biggr| + \biggr|\bar\mu_{A_t,t}-\frac{\sum_{s=1}^{t-1}\mu_{A_t,s}}{t-s'_i}\biggr| \\
        \le& (\sqrt{C_1} + \sqrt{2})\cdot \sqrt{\frac{\ln(KT^3)}{t-s'_i}} + 2\sum_{j\le\Gamma}\cV_{\cI_j}\I\{t\in\cI_j\}.
    \end{align*}
    For $t = s'_i$, it is easy to bound $\bar\mu_{a,s'_i} - \bar\mu_{A_t,s'_i} \le 1$. So we have
    \begin{align*}
    \left(\bar\mu^*_t - \bar\mu_{A_t,t}\right)\cdot\I\left\{\cN^s\land \cC^{D}_t\land \lnot \cP_t\right\}
    \le& \min\left\{(\sqrt{C_1} + \sqrt{2})\cdot \sqrt{\frac{\ln(KT^3)}{t-s'_i}},1\right\}\\
    &\ ++2\sum_{j\le\Gamma}\cV_{\cI_j}\I\{t\in\cI_j\}
    \end{align*}
    Sum up all $t$, we have
    \begin{align*}
    \left(\bar\mu^*_t - \bar\mu_{A_t,t}\right)\cdot\I\left\{\cN^s\land \cC^{D}_t\land \lnot \cP_t\right\}
    \le& \Gamma' + 2(\sqrt{C_1} + \sqrt{2})\sqrt{\ln(KT^3)}\sqrt{\Gamma'T} + 2\sum_{j=1}^{\Gamma}|\cI_j|\cdot\cV_{\cI_j}\\
    \le& 2\Gamma + 2(\sqrt{C_1} + \sqrt{2})\sqrt{\ln(KT^3)}\sqrt{2\Gamma T} + 2\sum_{j=1}^{\Gamma}|\cI_j|\cV_{\cI_j}.
    \end{align*}
    Then we conclude the proof of this lemma.
\end{proof}

\begin{restatable}{lemma}{lemsecondthirdtermdynamic}\label{lem:secondthird-term-dynamic}
    \begin{align*}
    &\E\left[\sum_{t=1}^T\left(\bar\mu^*_t - \bar\mu_{A_t,t}\right)\cdot\I\left\{\cN^s\land \cC^{D}_t\land \cP_t\right\}\right] +\E\left[\sum_{t=1}^T\left(\bar\mu^*_t - \bar\mu_{A_t,t}\right)\cdot\I\left\{\cN^s\land \lnot \cC^{D}_t\land\lnot \cB^D_t\right\}\right] \\
    \le& \left(C_2\ln (KT^3)\sqrt{(\Gamma+1)T} + 2\sqrt{\frac{\Gamma+1}{T}}\right) \times \left(3-\log_2 \sqrt{\frac{C_1\ln(KT^3)}{T}}\right).
    \end{align*}
\end{restatable}

\begin{proof}
    All of the arguments of this proof are the same as those in the proof of Lemma \ref{lem:secondthird-term-switching}, instead of changing $\mu_{a,t}$ into $\bar\mu_{a,t}$ for any arm $a$ and any time $t$, changing $\cC_t$ and $\cB_t$ into $\cC^{D}_t$ and $\cB^D_t$, and changing $\cS$ into $\Gamma$.
\end{proof}

\begin{restatable}{lemma}{lemforthtermdynamic}\label{lem:forth-term-dynamic}
    \begin{align*}
    \E\left[\sum_{t=1}^T\left(\bar\mu^*_t - \bar\mu_{A_t,t}\right)\cdot\I\left\{\cN^s\land \lnot \cC^{D}_t\land \cB^D_t\right\}\right]
    \le& 24\sqrt{(\Gamma +1)T}+24\sqrt{C_2\ln(KT^3)\Gamma T}.
    \end{align*}
\end{restatable}

\begin{proof}
    Recall that we have the partition of the interval $[1,T]$ into $\cI_1\cup\cI_2\cup\cdots\cup\cI_{\Gamma}$ such that $\cI_i \cap \cI_j = \phi$, and for any $i\le\Gamma$, $\cV_{\cI_i} \le \sqrt{\frac{C_3}{|\cI_i|}}$, and $\Gamma \le \left(\frac{2T}{C_3}\right)^{1/3}\cV^{\frac{2}{3}}+ 1$. For convinience, we will use $\I_t$ to denote the variable $\I\left\{\cN^s\land \lnot \cC^{D}_t\land \cB^D_t\right\}$.
    
    Similar to the proof of Lemma \ref{lem:forth-term-switching}, we will use induction to show the lemma. Then we set $R_{j}^{\ell}$ to be the following quantity
    \[\E_{s_{j}}\left[\sum_{t=s_{\ell}}^T\left(\bar\mu^*_t - \bar\mu_{A_t,t}\right)\cdot\I_t\bigg|\tau_{\ell} < s_j \le \tau_{\ell+1}\right],\]
    and $R^{\ell}(s)$ to be the following quantity 
    \[\E_s\left[\sum_{t = s}^T\left(\bar\mu^*_t - \bar\mu_{A_t,t}\right)\cdot\I_t\bigg|\tau_{\ell} = s\right],\]
    where $\E_s$ means that we take the expectation of all the randomness from time $s$ to time $T$. 
    
    We will use induction to show that
    \begin{align*}
    R_{j}^{\ell}\le& 12\sum_{k=\ell}^{\Gamma}\sqrt{\frac{T}{k+1}}+24\sum_{j'\ge j}\sqrt{C_2\ln(KT^3)(e_{j'}-s_{j'}+1)}. 
    \end{align*}
    We also divide the proof into 3 parts. In the first part we show the induction is correct in the base cases. In the second part we relate $R_{j}^{\ell}$ and $R^{\ell}(s)$. In the third part we finish the induction step using the results in the second part. For convinience, if the proof is the same as or very similar to the corresponding part in the proof of Lemma \ref{lem:forth-term-switching}, we will omit them in this proof.
    
    The induction has two dimensions $j$ and $\ell$, where $j$ denotes the index of the interval and $\ell$ denote the index of the starting epoch, and the induction proceeds from back to front.
    
    \textbf{Part 1:} The base cases are ture given Lemma \ref{lem:number-of-epochs-dynamic}. The arguments are very similar to that in the proof of Lemma \ref{lem:forth-term-switching}.
    
    \textbf{Part 2:} Same as before, let $\omega_s$ denote a possible realization of all the randomness before time $s$. Next we show that, if for $j_0,\ell_0$, such that for any $\omega_{s_{j_0}}$ satisfies $\tau_{\ell_0} < s_{j_0} \le \tau_{{\ell_0}+1}$, we have
    \begin{align*}
    R_{j_0}^{\ell_0}(\omega_{s_{j_0}})\le& 12\sum_{k=\ell_0}^{\Gamma}\sqrt{\frac{T}{k+1}}+24\sum_{j'\ge j_0}\sqrt{C_2\ln(KT^3)(e_{j'}-s_{j'}+1)},
    \end{align*}
    then for any $s\in[s_{j_0-1},e_{j_0-1}]$ and any $\omega_s$ satisfied $\tau_{\ell_0} = s$, we have
    \begin{align*}
    R^{\ell_0}(s)(\omega_{s})\le& 12\sum_{k=\ell_0}^{\Gamma}\sqrt{\frac{T}{k+1}}+24\sum_{j'\ge j_0}\sqrt{C_2\ln(KT^3)(e_{j'}-s_{j'}+1)}.
    \end{align*}
    From the additive property of the conditional expectation, we have
    \begin{align*}
    R^{\ell_0}(s)=& \E_s\left[\sum_{t = s}^T\left(\bar\mu^*_t - \bar\mu_{A_t,t}\right)\cdot\I_t\bigg|\tau_{\ell_0} = s\right]\\
    =& \E_s\left[\underbrace{\sum_{t = s_{j_0}}^{T}\left(\bar\mu^*_t - \bar\mu_{A_t,t}\right)\cdot\I_t}_{\text{term } \mathbb A}\bigg|\tau_{\ell_0} = s\right]\\
    +& \E_s\left[\underbrace{\sum_{t = s}^{s_{j_0-1}}\left(\bar\mu^*_t - \bar\mu_{A_t,t}\right)\cdot\I_t}_{\text{term } \mathbb B}\bigg|\tau_{\ell_0} = s\right]
    \end{align*}
    
    Similar to the corresponding part in the proof of Lemma \ref{lem:forth-term-switching}, we also have
    \begin{align*}
    \E_{s_{j_0}}\left[\mathbb A\big|\tau_{\ell_0} = s\right]\le 12\sum_{k=\ell_0}^{\Gamma}\sqrt{\frac{T}{k+1}}
    +24\sum_{j'\ge j_0}\sqrt{C_2\ln(KT^3)(e_{j'}-s_{j'}+1)}.
    \end{align*}
    Then we show that
    \[\E_{s_{j_0}}\left[\mathbb B\big|\tau_{\ell_0} = s\right] = 0.\]
    The proof is also similar to the corresponding part, but now we have to consider the contribution of the variation. We show that if $\cN^s$ happens and epoch $\ell$ start at $s\in [s_{j_0-1},e_{j_0-1}]$, then $\I\left\{\cC^{D}_t\right\} = 1$ for all $s\le t\le e_{j_0-1}$. This is due to the fact that if we add $\bv$ into $\cM$ at time $t'$ such that $s\le t'\le e_{j_0-1}$, then we know that $\exists a> e$ and $\sigma\ge s$,
    \[\hat\mu_a[\sigma,t'] - \hat\mu_{e}[\sigma,t'] > \sqrt{\frac{C_1\ln(KT^3)}{t'+1-\sigma}},\]
    and we have
    \[g = \hat\mu_a[\sigma,t'] - \hat\mu_{e}[\sigma,t'] > \sqrt{\frac{C_1\ln(KT^3)}{t'+1-\sigma}}. \]
    If $\cN^s$ happens, we know that for any arm $b\ge e$, we have
    \begin{align}
    |\hat\mu_b[\sigma,t'] - \mu_{b,t'}|
    \le& \biggr|\hat\mu_b[\sigma,t'] - \frac{\sum_{s=\sigma}^{t'}\mu_{b,s}}{t'+1-\sigma}\biggr| +\biggr|\bar\mu_b[\sigma,t'] - \frac{\sum_{s=\sigma}^{t'}\mu_{b,s}}{t'+1-\sigma}\biggr|\nonumber\\
    <&\sqrt{\frac{\ln(KT^3)}{2(t'+1-\sigma)}} + \sqrt{\frac{C_3}{|\cI_{j_0-1}|}}\nonumber \\
    \le& 2\sqrt{\frac{\ln(KT^3)}{2(t'+1-\sigma)}}\nonumber\\
    \le&\frac{g}{8},\label{equ:no-bad-procedure-in-same-interval-dynamic}
    \end{align}
    where we suppose that $t'$ is in $\cI_{j_0-1}$. The third inequality is due to the fact that $t'$ and $\sigma$ both in $\cI_{j_0-1}$ and $C_3 \le \frac{1}{2}$. The last inequality comes from the fact that we choose $C_2 \ge 128$. Then we finish the proof of \textbf{Part 2}.
    
    \textbf{Part 3:} The proof is largely the same as the \textbf{Part 3} in proof of Lemma, $\cC_t$ into $\cC^{D}_t$, and $\cB_t$ into $\cB^D_t$. However, we have to consider the contribution of the variation in the key observation in \textbf{Part 3} and the equation is different(changing Equation \ref{equ:key-observation-switching} into \ref{equ:key-observation-dynamic}). For completeness, we copy the proof and do some small modification.
    
    In this part, we complete the induction step. Suppose that the inequality holds for $R_{j+1}^{\ell}$ and $R_{j+1}^{\ell+1}$, we now show that the inequality also hold for $R_{j}^{\ell}$. First from \textbf{Part 2}, we know that the inequality holds for $R^{\ell+1}(s)$ for all $s\in [s_j,e_j]$.
    
    Given $\omega_{s_j}$, we know that $\cM_{s_j}$ is determined. Then we will show that for any $s_j \le t\le e_j$, we have $\I\left\{\cC^D_t\land\cN^s\right\} = \I\left\{\cC^{D}_{s_j}\land\cN^s\right\}$. We can first assume that $\cN^s$ happens, otherwise both sides of the equation is 0. Then if a mean vector $\bv$ is added into the set $\cM$ by our algorithm during time interval $[s_j,e_j]$, and if we assume that $\cN^s$ happens, then from Equation \ref{equ:no-bad-procedure-in-same-interval-dynamic}, we know that $\bv$ will not affect whether $\cC^{D}_t$ will happen for any $t\in [s_j,e_j]$. In this way, only the vectors $\bv'\in\cM_{s_j}$ affect $\cC_t$. Then, we have $\I\left\{\cC^D_t\land\cN^s\right\} = \I\left\{\cC^{D}_{s_j}\land\cN^s\right\}$ for all $s_j \le t\le e_j$.
    
    Now we can suppose that $\I\left\{\cC^{D}_{s_j}\land\cN^s\right\} = 1$, otherwise the induction is absolutely true(The steps $s_j \le t\le e_j$ do not contribute to the filtered regret). Then we assume arm $b_{s_j}$ to be the smallest arm such that $\exists (g,e,\bv)$ such that $e = b_{s_j}$ and exists $a\ge b_{s_j},|\bv_{a}-\bar\mu_{a,s_j}|>\frac{g}{4}$. From now on, we use $(g^j,e^j,\bv^j)$ to denote the tuple such that $e^j = b_{s_j}$. We set $\varepsilon_j = \max_{a\ge b_{s_j}}|\bar\mu_{a,s_j} - \bv^j_{a}|$ to be the max difference from the average mean in an interval and the estimated mean. We set $d_{i_j} < \varepsilon_j \le 2d_{i_j}$ where $d_i = 2^{-i}$ is defined in the algorithm. 
    
    Then there are 2 kinds of arms during time $s_j \le t\le e_j$, the first is that $A_t < b_{s_j}$, which only happens in an exploration phase, and the second is that $A_t \ge b_{s_j}$, which may happens both in an exploration phase or not in any exploration phase. However, we can know that $b_t$ are all the same for any $s_j \le t\le e_j$, since if we add $\bv$ into $\cM$ during time interval $[s_j,e_j]$, then $\bv$ will not affect whether $\cC^{D}_t$ happens during time interval $[s_j,e_j]$, so $\cB_t^D = \{A_t \ge b_t\} = \{A_t \ge b_{s_j}\}$.
    
    First, we have
    \[d_{i_j} \ge \frac{\varepsilon_j}{2} \ge \frac{g^j}{8},\]
    and thus for $A_t \ge b_{s_j}$, we have
    \begin{equation}\label{equ:regret-bad-arm-dynamic}
    \bar\mu_t^* - \bar\mu_{A_t,t} \le 2\varepsilon_j + g^j \le 6\varepsilon_j \le 12 d_{i_j},
    \end{equation}
    since all the arms $a < b_{s_j}$ will not have optimal average mean in the interval(from the definition of $b_{s_j}$ and the definition of $\cC^{D}_{s_j}$).
    
    Next we decompose the regret $R_{j}^{\ell}(\omega_{s_j})$ into several cases: $\tau_{\ell+1} = s$ for $s\in [s_j,e_j]$ or $\tau_{\ell} < s_{j+1} \le \tau_{\ell+1}$. In the first case, we use $R^{\ell+1}(s)$ to complete the induction and in the second case, we use $R_{j+1}^{\ell}$ to complete the induction. Then we have for any $\omega_{s_j}$ satisfies $\tau_{\ell} < s_j \le \tau_{\ell+1}$.
    \begin{align}
    R_{j}^{\ell}(\omega_{s_j})
    =& \E_{s_{j}}\left[\sum_{t=s_{\ell}}^T\underbrace{\left(\bar\mu^*_t - \bar\mu_{A_t,t}\right)\cdot \I\left\{\cN^s\land \lnot \cC^{D}_t\land\cB^D_t\right\}}_{\text{term }\mathbb C_t}\biggr|\omega_{s_j}\right]\nonumber \\
    =& \sum_{s=s_j}^{e_j}\E_{s_{j}}\left[\sum_{t=s_{\ell}}^T\mathbb C_t\cdot\I\{\tau_{\ell+1} = s\}\biggr|\omega_{s_j}\right]+ \E_{s_{j}}\left[\sum_{t=s_{\ell}}^T\mathbb C_t\cdot\I\{\tau_{\ell} <s_{j+1} \le \tau_{\ell+1} \}\biggr|\omega_{s_j}\right]\nonumber \\
    \le& \sum_{s=s_j}^{e_j}\E_{s_{j}}\left[\sum_{t=s_{\ell}}^{s-1}\mathbb C_t\cdot\I\{\tau_{\ell+1} = s\}\biggr|\omega_{s_j}\right]\label{equ:induction-illus-dynamic}+ \sum_{s=s_j}^{e_j}\Pr\{\tau_{\ell+1} = s\land \cN^s|\omega_{s_j}\}\cdot ||R^{\ell+1}(s)||_{\infty} \\
    &\ + \E_{s_{j}}\left[\sum_{t=s_{j}}^{e_j}\mathbb C_t\cdot\I\{\tau_{\ell} <s_{j+1} \le \tau_{\ell+1} \}\biggr|\omega_{s_j}\right] + \Pr\{\tau_{\ell} <s_{j+1}\le \tau_{\ell+1}\land \cN^s|\omega_{s_j}\}\cdot ||R_{j+1}^{\ell}||_{\infty}\nonumber \\
    \le& \sum_{s=s_j}^{e_j}\E_{s_{j}}\left[\sum_{t=s_{\ell}}^{s-1}\mathbb C_t\cdot\I\{\tau_{\ell+1} = s\}\biggr|\omega_{s_j}\right] + \E_{s_{j}}\left[\sum_{t=s_{j}}^{e_j}\mathbb C_t\cdot\I\{\tau_{\ell} <s_{j+1} \le \tau_{\ell+1} \}\biggr|\omega_{s_j}\right]\nonumber  \\
    &\ + \Pr\{\tau_{\ell} <s_{j+1}\le \tau_{\ell+1}\land \cN^s|\omega_{s_j}\}12\sqrt{\frac{T}{\ell+1}}\nonumber \\
    &\ +12\sum_{k=\ell+1}^{\Gamma}\sqrt{\frac{T}{k+1}}+24\sum_{j'\ge j+1}\sqrt{C_2\ln(KT^3)(e_{j'}-s_{j'}+1)},\nonumber 
    \end{align}
    where the equation (\ref{equ:induction-illus-dynamic}) comes from the fact that for any $\omega$ such that $\lnot\cN^s$ happens, $R^{\ell+1}(s)(\omega) = R^{\ell}_{j+1}(\omega) = 0$, and the last 3 terms comes from the induction step. Then to complete the induction step, we just have to show that
    \begin{align}
    & \sum_{s=s_j}^{e_j}\E_{s_{j}}\left[\sum_{t=s_{\ell}}^{s-1}\mathbb C_t\cdot\I\{\tau_{\ell+1} = s\}\biggr|\omega_{s_j}\right]\nonumber\\
    &\ + \E_{s_{j}}\left[\sum_{t=s_{j}}^{e_j}\mathbb C_t\cdot\I\{\tau_{\ell} <s_{j+1} \le \tau_{\ell+1} \}\biggr|\omega_{s_j}\right]\nonumber \\
    &\ + \Pr\{\tau_{\ell} <s_{j+1}\le \tau_{\ell+1}\land \cN^s|\omega_{s_j}\}12\sqrt{\frac{T}{\ell+1}}\nonumber\\
    \le& 12\sqrt{\frac{T}{\ell+1}} + 24\sqrt{C_2\ln(KT^3)(e_{j}-s_{j}+1)}\label{equ:induction-dynamic}.
    \end{align}
    We first observe that from Equation (\ref{equ:regret-bad-arm-dynamic}),
    \begin{align*}
    \mathbb C_t =& \left(\bar\mu^*_t - \bar\mu_{A_t,t}\right)\cdot \I\left\{\cN^s\land \lnot \cC^{D}_t\land\cB^D_t\right\}\\
    \le& 12 d_{i_j}\cdot \I\left\{\cN^s\land \lnot \cC^{D}_t\land\cB^D_t\right\}.
    \end{align*}
    Then if the length of interval $[s_j,e_j]$ satisfies
    \[s_j-e_j + 1 \le 2\left\lceil\frac{C_2\ln (KT^3)}{d_{i_j}^2}\right\rceil,\]
    then when $KT^3>1$, we have
    \begin{align*}
    d_{i_j} \le \sqrt{\frac{4C_2\ln (KT^3)}{s_j-e_j+1}},
    \end{align*}
    and thus we can show that
    \begin{align*}
    & \sum_{s=s_j}^{e_j}\E_{s_{j}}\left[\sum_{t=s_{\ell}}^{s-1}\mathbb C_t\cdot\I\{\tau_{\ell+1} = s\}\biggr|\omega_{s_j}\right]+ \E_{s_{j}}\left[\sum_{t=s_{j}}^{e_j}\mathbb C_t\cdot\I\{\tau_{\ell} <s_{j+1} \le \tau_{\ell+1} \}\biggr|\omega_{s_j}\right] \\
    \le& 12d_{i_j}(s_j-e_j+1)\\
    \le& 24\sqrt{C_2\ln (KT^3)(s_j-e_j+1)},
    \end{align*}
    and we finish the induction step in this case using the fact that
    \[\Pr\{\tau_{\ell} <s_{j+1}\le \tau_{\ell+1}\land \cN^s|\omega_{s_j}\} \le 1.\]
    Then we just have to finish the induction step in the second case where
    \[s_j-e_j + 1 > 2\left\lceil\frac{C_2\ln (KT^3)}{d_{i_j}^2}\right\rceil \ge \frac{2C_2\ln (KT^3)}{d_{i_j}^2}. \]
    Then the key observation is that, when $\cN^s$ happens and an exploration phase $\left(d_{i_j},\left[t,t+\lceil\frac{C_2\ln (KT^3)}{d_{i_j}^2}\rceil\right)\right)$ is inserted into $\cM$ in the time interval $t\in \left[s_j,e_j-\left\lceil\frac{C_2\ln (KT^3)}{d_{i_j}^2}\right\rceil+1\right]$, the non-stationary test will detect the non-stationary at $t' = t+\lceil\frac{C_2\ln (KT^3)}{d_{i_j}^2}\rceil-1$ and the new epoch will start at time $t'+1$. This is due to the fact that: As $d_{i_j} \ge \frac{\varepsilon_j}{2} \ge \frac{g^j}{8}$ and $(g^j,e^j,\bv^j)$ is in $\cM$ at time $s_j$, then from Lemma \ref{lem:exploration-observation}, we know that all the arms $a\ge b_{s_j}$ are observed for $\lceil\frac{C_2\ln (KT^3)}{d_{i_j}^2}\rceil$ times. Then for the arm $a$ such that $|\bar\mu_{a,s_j} -\bv^j_{a}| = \varepsilon_j > d_{i_j}$, we know that by $\cN^s$, we have
    \begin{align}
        |\bar\mu_{a,s_j} -\hat\mu_a\left[t,t'+1\right)|
         \le&\biggr|\bar\mu_{a,s_j} - \frac{\sum_{s=t}^{t'}\mu_{a,s}}{t'-t+1}\biggr| + \biggr|\frac{\sum_{s=t}^{t'}\mu_{a,s}}{t'-t+1} -\hat\mu_a\left[t,t'+1\right)\biggr|\nonumber\\
         \le& \sqrt{\frac{C_3}{|\cI_{j}|}} + \sqrt{\frac{\ln(KT^3)}{2(t'-t+1)}}\nonumber \\
         \le& 2\sqrt{\frac{\ln(KT^3)}{2(t'-t+1)}}\nonumber \\
         \le& \frac{2d_{i_j}}{\sqrt{2C_2}}\nonumber\\
         \le& \frac{d_{i_j}}{8},\label{equ:key-observation-dynamic}
    \end{align}
    and thus we can conclude
    \[|\bv^j_{a} - \hat\mu_a\left[t,t'+1\right)| > \frac{7d_{i_j}}{8} \ge \frac{d_{i_j}}{4},\]
    and the epoch will end at time $t'$. Then from the key observation before, we have
    \[\Pr\{\tau_{\ell} <s_{j+1}\le \tau_{\ell+1}\land \cN^s|\omega_{s_j}\} \le (1-p_{\ell,i_j})^{e_j+1-s_j-w_{i_j}},\]
    where $p_{\ell,i} = d_i\sqrt{\frac{\ell+1}{T}}$ is the probability to add a sampling phase with index $i$ at round $\ell$(see Algorithm \ref{alg:nonstationaryelim}) and $w_{i} = \left\lceil\frac{C_2\ln (KT^3)}{d_{i}^2}\right\rceil$ is a shorthand of the length of the sampling phase with index $i$.
    
    Then we bound
    \begin{align}
    & \sum_{s=s_j}^{e_j}\E_{s_{j}}\left[\sum_{t=s_{\ell}}^{s-1}\mathbb C_t\cdot\I\{\tau_{\ell+1} = s\}\biggr|\omega_{s_j}\right]+ \E_{s_{j}}\left[\sum_{t=s_{j}}^{e_j}\mathbb C_t\cdot\I\{\tau_{\ell} <s_{j+1} \le \tau_{\ell+1} \}\biggr|\omega_{s_j}\right]\label{equ:left-term-dynamic}.
    \end{align}
    As $\mathbb C_t = \left(\bar\mu^*_t - \bar\mu_{A_t,t}\right)\cdot \I\left\{\cN^s\land \lnot \cC_t^D\land\cB_t^D\right\}$ and when $\cB_t^D$ happens, from Equation (\ref{equ:regret-bad-arm-dynamic}), we have $\bar\mu_t^* - \bar\mu_{A_t,t} \le 12d_{i_j}$. Then we can bound (\ref{equ:left-term-dynamic}) by
    \begin{align*}
    & \sum_{s=s_j}^{e_j}\E_{s_{j}}\left[\sum_{t=s_{\ell}}^{s-1}12d_{i_j}\cdot\I\{\tau_{\ell+1} = s\land\cN^s\}\biggr|\omega_{s_j}\right]\\
    &\ + \E_{s_{j}}\left[\sum_{t=s_{j}}^{e_j}12d_{i_j}\cdot\I\{\tau_{\ell} <s_{j+1} \le \tau_{\ell+1} \land\cN^s\}\biggr|\omega_{s_j}\right]\\
    =& 12d_{i_j}\sum_{s=s_j}^{e_j}\Pr\left\{\tau_{\ell+1}> s\land\cN^s\big|\omega_{s_j}\right\}.
    \end{align*}
    For $s\le s_j + w_{i_j}-1$, we bound $\Pr\left\{\tau_{\ell+1}> s\land\cN^s\big|\omega_{s_j}\right\} \le 1$, and for $s\ge s_j + w_{i_j}$, we bound 
    \[\Pr\left\{\tau_{\ell+1}> s\land\cN^s\big|\omega_{s_j}\right\} \le (1-p_{\ell,i_j})^{s+1-s_j-w_{i_j}}.\]
    Then we have
    \begin{align*}
    12d_{i_j}\sum_{s=s_j}^{e_j}\Pr\left\{\tau_{\ell+1}> s\land\cN^s\big|\omega_{s_j}\right\}
    \le& 12d_{i_j}w_{i_j} + 12d_{i_j}\sum_{r=1}^{e_j+1-s_j-w_{i_j}}(1-p_{\ell,i_j})^{r}\\
    \le& 12d_{i_j}\left[w_{i_j} + \frac{1-(1-p_{\ell,i_j})^{e_j+1-s_j-w_{i_j}}}{p_{\ell,i_j}}\right]\\
    =& 12d_{i_j}w_{i_j} + 12\frac{d_{i_j}}{p_{\ell,i_j}}\left[1-(1-p_{\ell,i_j})^{e_j+1-s_j-w_{i_j}}\right]\\
    \le& 12\sqrt{\frac{T}{\ell+1}}\left[1-(1-p_{\ell,i_j})^{e_j+1-s_j-w_{i_j}}\right] + 24\frac{C_2\ln(KT^3)}{d_{i_j}}\\
    \le& 12\sqrt{\frac{T}{\ell+1}}\left[1-(1-p_{\ell,i_j})^{e_j+1-s_j-w_{i_j}}\right] + 24\sqrt{C_2\ln (KT^3)(s_j-e_j+1)}.
    \end{align*}
    Then we finish the induction(Equation \ref{equ:induction-dynamic}) by the previous argument that
    \[\Pr\{\tau_{\ell} <s_{j+1}\le \tau_{\ell+1}\land \cN^s|\omega_{s_j}\} \le (1-p_{\ell,i_j})^{e_j+1-s_j-w_{i_j}}.\]
    
    Given the induction result, we conclude the proof by
    \begin{align*}
    R^{1}(1)\le& 12\sum_{k=1}^{\Gamma}\sqrt{\frac{T}{k+1}}+24\sum_{j'\ge 1}\sqrt{C_2\ln(KT^3)(e_{j'}-s_{j'}+1)}\\
    \le& 24\sqrt{(\Gamma+1)T}+24\sqrt{C_2\ln(KT^3)\Gamma T}.
    \end{align*}
\end{proof}

\begin{restatable}{lemma}{lemfifthdynamic}\label{lem:fifth-term-dynamic}
    \[\E\left[\sum_{t=1}^T\left(\bar\mu^*_t - \bar\mu_{A_t,t}\right)\cdot\I\left\{\lnot\cN^s\right\}\right] \le 2.\]
\end{restatable}

\begin{proof}
    From Lemma \ref{lem:samplingnice}, we know that $\Pr\{\lnot\cN^s\} \le \frac{2}{T}$. Then, since the distribution $\nu_t$ has support on $[0,1]^K$, we know that $|\bar\mu_{A_t,t} - \bar\mu^*_t| \le 1$, and we have
    \begin{align*}
    \E\left[\sum_{t=1}^T\left(\bar\mu^*_t - \bar\mu_{A_t,t}\right)\cdot\I\left\{\lnot\cN^s\right\}\right]
    \le& \E\left[\sum_{t=1}^T\I\left\{\lnot\cN^s\right\}\right] \\
    =& T\cdot \E\left[\I\left\{\lnot\cN^s\right\}\right] \\
    =& T\cdot \Pr\{\lnot\cN^s\}\\
    \le& 2.
    \end{align*}
\end{proof}

Then, we have the following lemma to conclude the proof of Theorem \ref{thm:dynamic-regret}.

\begin{restatable}{lemma}{lemsixthtermdynamic}\label{lem:sixth-term-dynamic}
    \[\sum_{j=1}^{\Gamma}|\cI_j|\cdot\cV_{\cI_j} \le \sqrt{C_3\Gamma T}.\]
\end{restatable}

\begin{proof}
    \begin{align*}
        \sum_{j=1}^{\Gamma}|\cI_j|\cdot\cV_{\cI_j} = \sum_{j=1}^{\Gamma}|\cI_j|\cdot\sqrt{\frac{C_3}{|\cI_j|}} \le& \sqrt{C_3\Gamma T},
    \end{align*}
    where we use the fact that $\sum_{j=1}^{\Gamma}|\cI_j| = T$.
\end{proof}

\thmdynamicregret*

\begin{proof}[Proof of Theorem \ref{thm:dynamic-regret}]
    The proof follows directly from Lemma \ref{lem:first-term-dynamic},\ref{lem:secondthird-term-dynamic},\ref{lem:forth-term-dynamic},\ref{lem:fifth-term-dynamic},\ref{lem:sixth-term-dynamic}, and the fact that 
    \[\Gamma \le \left(\frac{2T}{C_3}\right)^{1/3}\cV^{\frac{2}{3}}+ 1.\]
\end{proof}

Note that the parameter $C_3$ appears only in the proof but not in the algorithm and theorem.

\section{Proof of Theorem \ref{thm:regret-auction} }
\thmregretauction*

Given the 2 regret bound in for the one sided full information bandit, this theorem is easy to prove. We will first formally show that the second price auction is one-sided Lipschitz, so the discretization will only induce small regret. Then, we can directly show that $\text{Reg}_{\cA}^{SP} = \tilde\cO(\sqrt{\cS T})$ from Theorem \ref{thm:switching-regret}. To prove that $\text{Reg}_{\cA}^{SP} = \tilde\cO(\bar\cV^{\frac{1}{3}}T^{\frac{2}{3}})$, we need to further relate the sum of total variation $\bar\cV$ with the variation of the means $\cV$ in the one-sided bandit problem.

\begin{proof}[Proof of Theorem \ref{thm:regret-auction}]
    First, we show that given any 2 reserve price $0\le r_1\le r_2\le 1$ and any private value distribution $\cD$, we have
    \[\cR(r_1,\cD) \ge \cR(r_2,\cD) - (r_2 - r_1).\]
    Plug in the definition of revenue, we only have to show that
    \[\E_{\bv\sim\cD}\left[\sum_{i=1}^n p_i(r_2,\bv) - \sum_{i=1}^n p_i(r_1,\bv)\right] \le r_2 - r_1.\]
    This is true because in the second price auction, only the bidder with largest private value \emph{may} pay, and for any private value vector $\bv$, the payment difference is at most $r_2 - r_1$, i.e.
    \[\sum_{i=1}^n p_i(r_2,\bv) - \sum_{i=1}^n p_i(r_1,\bv) \le r_2 - r_1.\]
    Then any online auction problem(Definition \ref{def:non-stationary-auction}) corresponds to an one-sided full-information bandit problem(Definition \ref{def:non-stationary-bandit}) with number of arms $\lceil\sqrt{T}\rceil+1$. This is due to the fact that given any private value vector $\bv$, we have that for any $0\le k\le \lceil\sqrt{T}\rceil$, the reward of $r_k$ is $\sum_{i=1}^n p_i(r_k,\bv)$, and $\E_{\bv\sim\cD}[\sum_{i=1}^n p_i(r_k,\bv)] = \cR(r_k,\cD)$.
    
    Then, we have
    \begin{align*}
        \text{Reg}_{\cA}^{SP} =& \E\left[\sum_{t=1}^T (\cR(r^*_t,\cD_t) - \cR(r^{(t)},\cD_t))\right]\\
        =& \E\left[\sum_{t=1}^T (\cR(r^*_t,\cD_t) - \cR(\bar r^*_t,\cD_t))\right]+ \E\left[\sum_{t=1}^T (\cR(\bar r^*_t,\cD_t) - \cR(r^{(t)},\cD_t))\right],
    \end{align*}
    where $r^*_t := \argmax_{r_k}\cR(\bar r_k,\cD_t)$ is the best reserve price in the discrete domain. We first show that
    \[\E\left[\sum_{t=1}^T (\cR(r^*_t,\cD_t) - \cR(\bar r^*_t,\cD_t))\right] \le \sqrt{T}.\]
    This is because from the previous argument of the one-sided Lipschitz condition, and if we define $\hat r^*_t = r_j$ where $r_j \le r^*_t < r_{j+1}$, we have
    \begin{align*}
        \E\left[\sum_{t=1}^T (\cR(r^*_t,\cD_t) - \cR(\bar r^*_t,\cD_t))\right]
        \le&\E\left[\sum_{t=1}^T (\cR(r^*_t,\cD_t) - \cR(\hat r^*_t,\cD_t))\right] \\
        \le&\E\left[\sum_{t=1}^T (r^*_t - \hat r^*_t)\right] \\
        \le& T \cdot \frac{1}{\lceil\sqrt{T}\rceil} \\
        \le& \sqrt{T}.
    \end{align*}
    Then because the number of switchings $\cS$ in the auction problem is the same as the number of switchings in the one-sided bandit case, and from Theorem \ref{thm:switching-regret}, when using our algorithm $\Elim$, we have
    \[\E\left[\sum_{t=1}^T (\cR(\bar r^*_t,\cD_t) - \cR(r^{(t)},\cD_t))\right] = \tilde\cO(\sqrt{\cS T}).\]
    Then we have $\text{Reg}_{\cA}^{SP} = \tilde\cO(\sqrt{\cS T})$. To show that $\text{Reg}_{\cA}^{SP} = \tilde\cO(\bar\cV^{\frac{1}{3}}T^{\frac{2}{3}})$, we need to bound the variation in the one-sided bandit case. We use $f_t$ to denote the probability distribution function of the distribution $\cD_t$. We have
    \begin{align*}
        \sum_{t=2}^T\max_{r_k}|\cR(r_k,\cD_t) - \cR(r_k,\cD_{t-1})|
        =&\sum_{t=2}^T\max_{r_k}\biggr|\int_{\bv\in [0,1]^n}(\sum_{i=1}^n p_i(r_k,\bv))(f_t(\bv) - f_{t-1}(\bv))d \bv\biggr| \\
        \le& \sum_{t=2}^T\int_{\bv\in [0,1]^n}|f_t(\bv) - f_{t-1}(\bv)|d \bv \\
        =& \sum_{t=2}^T 2||\cD_t - \cD_{t-1}||_{\text{TV}}.
    \end{align*}
    Then we know that the variation in the one-sided bandit case $\cV$ is bounded by $\cV \le 2\bar\cV$, and from Theorem \ref{thm:dynamic-regret}, we have
    \[\E\left[\sum_{t=1}^T (\cR(\bar r^*_t,\cD_t) - \cR(r^{(t)},\cD_t))\right] = \tilde\cO(\bar\cV^{\frac{1}{3}}T^{\frac{2}{3}}),\]
    and we conclude the proof of this theorem.
\end{proof}

\section{Proof of Theorem \ref{thm:lower-bound-switching}(Lower Bound)}
In this section, we show the proof of Theorem \ref{thm:lower-bound-switching}. The lower bound shows that our algorithm is nearly optimal(up to logarithm factors) in the switching case. Our proof is based on the following proposition in \cite{cesa2015regret} and its proof.

\proplowerboundstationary*

\thmlowerboundswitching*

\begin{proof}[Proof of Theorem \ref{thm:lower-bound-switching}]
    In the proof of Theorem 2 in \cite{cesa2015regret}, the authors construct 2 private value distributions: $\cD'_1$ and $\cD'_2$  where $\cD'_1$ has probability $\frac{1}{2} +\varepsilon$ to be $\frac{1}{2}$ and probability $\frac{1}{2}-\varepsilon$ to be $\frac{3}{4}$, and $\cD'_2$ has probability $\frac{1}{2} -\varepsilon$ to be $\frac{1}{2}$ and probability $\frac{1}{2}+\varepsilon$ to be $\frac{3}{4}$. Then suppose $T$ is large enough and choose $\varepsilon = \Theta(\sqrt{1/T})$. Then consider 2 i.i.d bidders with private value distribution $\cD'_1$ or $\cD'_2$, the author shows that for any deterministic algorithm, there exists a distribution $\cD'\in\{\cD'_1,\cD'_2\}$ such that if 2 i.i.d bidders have private value distribution following $\cD'$ and $T$ is large enough, the regret is $\Omega(\sqrt{T})$. It is easy to generalize the above statement from any deterministic algorithm into any algorithm by the Fubini's theorem.
    
    Then we first show how to prove the lower bound in the switching case. We choose $T' = \frac{T}{\cS}$ is large enough. Then we show that, for any algorithm, there exists an instance $\{\cD_t\}_{t\le T}\in \{\cD'_1,\cD'_2\}^T$ such that the dynamic regret is $\Omega(\sqrt{\cS T})$. We restrict the instance into the following form: for the distributions in a segment $t = kT'+1,\dots,(k+1)T'$, the distributions $\cD_t$ remain the same.
    
    We construct the distribution segment by segment. Suppose that we have construct the distributions in the previous $k$ segments, such that the switching regret in the previous $k$ segments are both $\Omega(\sqrt{T'})$, then we show that we can choose the distributions from $\{\cD'_1,\cD'_2\}$ such that the regret in segment $k+1$ is also $\Omega(\sqrt{T'})$. Otherwise, we can construct an algorithm such that it achieves static regret $o(\sqrt{T'})$ on both the distribution $\cD'_1$ and $\cD'_2$. The algorithm is to use the non-stationary algorithm we want to prove, and first executing on the distributions in the previous $k$ segment.
    
    Then by the proof of Proposition \ref{prop:lower-bound-stationary}, we know that there does not exist such algorithm, and we can choose the distribution in segment $k+1$ such that the non-stationary regret in segment $k+1$ is also $\Omega(\sqrt{T'})$. Then, the total non-stationary regret is $\Omega(\cS\sqrt{T'}) = \Omega(\sqrt{\cS T})$.
    
    Then we show the lower bound in the dynamic case. Given $\bar\cV$, we can first choose $\Delta$ large enough such that the regret on a segment with length $\Delta$ is at least $\Omega(\sqrt{\Delta})$. Note that with length $\Delta$, the variable $\varepsilon$ in the proof of Theorem 2 in \cite{cesa2015regret} is chosen to be $\Theta(\frac{1}{\sqrt{\Delta}})$. Then the total variation $||\cD'_1-\cD'_2||_{\text{TV}}$ is also bounded by $\Theta(\frac{1}{\sqrt{\Delta}})$. We choose the number of segments to be $\Theta(\bar\cV\sqrt{\Delta})$ and then we have the regret is lower bounded by $\Omega(\bar\cV\Delta)$. Note that $T = \Delta\cdot \bar\cV\sqrt{\Delta}$, solve for $\Delta$ we get $\Delta = \left(\frac{T}{\bar\cV}\right)^{\frac{2}{3}}$. Plug it into the lower bound before we get the dynamic regret is at least $\Omega(\bar\cV^{\frac{1}{3}}T^{\frac{2}{3}})$.
\end{proof}

%% file: arxiv_main.bbl
\begin{thebibliography}{10}

\bibitem{auer2002finite}
Peter Auer, Nicol{\`o} Cesa-Bianchi, and Paul Fischer.
\newblock Finite-time analysis of the multiarmed bandit problem.
\newblock {\em Machine Learning}, 47(2-3):235--256, 2002.

\bibitem{auer2002nonstochastic}
Peter Auer, Nicol{\`o} Cesa-Bianchi, Yoav Freund, and Robert~E. Schapire.
\newblock The nonstochastic multiarmed bandit problem.
\newblock {\em SIAM J. Comput.}, 32(1):48--77, 2002.

\bibitem{auer2O19adaptively}
Peter Auer, Pratik Gajane, and Ronald Ortner.
\newblock Adaptively tracking the best bandit arm with an unknown number of
  distribution changes.
\newblock In {\em Conference on Learning Theory, {COLT} 2019, 25-28 June 2019,
  Phoenix, AZ, {USA}}, pages 138--158, 2019.

\bibitem{besbes2015nonstationary}
Omar Besbes, Yonatan Gur, and Assaf~J. Zeevi.
\newblock Non-stationary stochastic optimization.
\newblock {\em Operations Research}, 63(5):1227--1244, 2015.

\bibitem{bubeck2012regret}
S{\'{e}}bastien Bubeck and Nicol{\`{o}} Cesa{-}Bianchi.
\newblock Regret analysis of stochastic and nonstochastic multi-armed bandit
  problems.
\newblock {\em Foundations and Trends in Machine Learning}, 5(1):1--122, 2012.

\bibitem{bianchi2017algorithmic}
Nicol{\`{o}} Cesa{-}Bianchi, Pierre Gaillard, Claudio Gentile, and
  S{\'{e}}bastien Gerchinovitz.
\newblock Algorithmic chaining and the role of partial feedback in online
  nonparametric learning.
\newblock In {\em Proceedings of the 30th Conference on Learning Theory, {COLT}
  2017, Amsterdam, The Netherlands, 7-10 July 2017}, pages 465--481, 2017.

\bibitem{cesa2015regret}
Nicolo Cesa-Bianchi, Claudio Gentile, and Yishay Mansour.
\newblock Regret minimization for reserve prices in second-price auctions.
\newblock {\em IEEE Transactions on Information Theory}, 61(1):549--564, 2015.

\bibitem{chen2019anewalgorithm}
Yifang Chen, Chung{-}Wei Lee, Haipeng Luo, and Chen{-}Yu Wei.
\newblock A new algorithm for non-stationary contextual bandits: Efficient,
  optimal and parameter-free.
\newblock In {\em Conference on Learning Theory, {COLT} 2019, 25-28 June 2019,
  Phoenix, AZ, {USA}}, pages 696--726, 2019.

\bibitem{cheung2019learning}
Wang~Chi Cheung, David Simchi-Levi, and Ruihao Zhu.
\newblock Learning to optimize under non-stationarity.
\newblock In Kamalika Chaudhuri and Masashi Sugiyama, editors, {\em Proceedings
  of Machine Learning Research}, volume~89 of {\em Proceedings of Machine
  Learning Research}, pages 1079--1087. PMLR, 16--18 Apr 2019.

\bibitem{garivier2011upper}
Aur{\'{e}}lien Garivier and Eric Moulines.
\newblock On upper-confidence bound policies for switching bandit problems.
\newblock In {\em Algorithmic Learning Theory - 22nd International Conference,
  {ALT} 2011, Espoo, Finland, October 5-7, 2011. Proceedings}, pages 174--188,
  2011.

\bibitem{gur2014stochastic}
Yonatan Gur, Assaf~J. Zeevi, and Omar Besbes.
\newblock Stochastic multi-armed-bandit problem with non-stationary rewards.
\newblock In {\em Advances in Neural Information Processing Systems 27: Annual
  Conference on Neural Information Processing Systems 2014, December 8-13 2014,
  Montreal, Quebec, Canada}, pages 199--207, 2014.

\bibitem{jun2017online}
Kwang{-}Sung Jun, Francesco Orabona, Stephen Wright, and Rebecca Willett.
\newblock Online learning for changing environments using coin betting.
\newblock {\em CoRR}, abs/1711.02545, 2017.

\bibitem{karnin2016multi}
Zohar~S. Karnin and Oren Anava.
\newblock Multi-armed bandits: Competing with optimal sequences.
\newblock In {\em Advances in Neural Information Processing Systems 29: Annual
  Conference on Neural Information Processing Systems 2016, December 5-10,
  2016, Barcelona, Spain}, pages 199--207, 2016.

\bibitem{liu2018change}
Fang Liu, Joohyun Lee, and Ness~B. Shroff.
\newblock A change-detection based framework for piecewise-stationary
  multi-armed bandit problem.
\newblock In {\em Proceedings of the Thirty-Second {AAAI} Conference on
  Artificial Intelligence, (AAAI-18), the 30th innovative Applications of
  Artificial Intelligence (IAAI-18), and the 8th {AAAI} Symposium on
  Educational Advances in Artificial Intelligence (EAAI-18), New Orleans,
  Louisiana, USA, February 2-7, 2018}, pages 3651--3658, 2018.

\bibitem{luo2015achieving}
Haipeng Luo and Robert~E. Schapire.
\newblock Achieving all with no parameters: Adanormalhedge.
\newblock In {\em Proceedings of The 28th Conference on Learning Theory, {COLT}
  2015, Paris, France, July 3-6, 2015}, pages 1286--1304, 2015.

\bibitem{luo2018efficient}
Haipeng Luo, Chen{-}Yu Wei, Alekh Agarwal, and John Langford.
\newblock Efficient contextual bandits in non-stationary worlds.
\newblock In {\em Conference On Learning Theory, {COLT} 2018, Stockholm,
  Sweden, 6-9 July 2018.}, pages 1739--1776, 2018.

\bibitem{mohri2015revenue}
Mehryar Mohri and Andres~Mu{\~{n}}oz Medina.
\newblock Revenue optimization against strategic buyers.
\newblock In {\em Advances in Neural Information Processing Systems 28: Annual
  Conference on Neural Information Processing Systems 2015, December 7-12,
  2015, Montreal, Quebec, Canada}, pages 2530--2538, 2015.

\bibitem{Myerson81}
Roger~B. Myerson.
\newblock Optimal auction design.
\newblock {\em Mathematics of Operations Research}, 6(1), feb 1981.

\bibitem{robbins1952bulletin}
Herbert Robbins.
\newblock Some aspects of the sequential design of experiments.
\newblock {\em Bull. Amer. Math. Soc.}, 58(5):527--535, 09 1952.

\bibitem{roughgarden2016minimizing}
Tim Roughgarden and Joshua~R. Wang.
\newblock Minimizing regret with multiple reserves.
\newblock In {\em Proceedings of the 2016 {ACM} Conference on Economics and
  Computation, {EC} '16, Maastricht, The Netherlands, July 24-28, 2016}, pages
  601--616, 2016.

\bibitem{thompson1933likelihood}
William~R Thompson.
\newblock On the likelihood that one unknown probability exceeds another in
  view of the evidence of two samples.
\newblock {\em Biometrika}, 25(3/4):285--294, 1933.

\bibitem{wei2016tracking}
Chen{-}Yu Wei, Yi{-}Te Hong, and Chi{-}Jen Lu.
\newblock Tracking the best expert in non-stationary stochastic environments.
\newblock In {\em Advances in Neural Information Processing Systems 29: Annual
  Conference on Neural Information Processing Systems 2016, December 5-10,
  2016, Barcelona, Spain}, pages 3972--3980, 2016.

\bibitem{zhang2018dynamic}
Lijun Zhang, Tianbao Yang, Rong Jin, and Zhi{-}Hua Zhou.
\newblock Dynamic regret of strongly adaptive methods.
\newblock In {\em Proceedings of the 35th International Conference on Machine
  Learning, {ICML} 2018, Stockholmsm{\"{a}}ssan, Stockholm, Sweden, July 10-15,
  2018}, pages 5877--5886, 2018.

\bibitem{zhao2019stochastic}
Haoyu Zhao and Wei Chen.
\newblock Stochastic one-sided full-information bandit.
\newblock {\em arXiv preprint arXiv:1906.08656}, 2019.

\end{thebibliography}
